\RequirePackage{fix-cm}
\documentclass[a4paper, 12pt, oneside]{llncs}
\usepackage{microtype}
\usepackage[T1]{fontenc}
\usepackage[utf8]{inputenc}
\usepackage{lmodern}
\usepackage[english]{babel}
\usepackage[babel]{csquotes}
\usepackage{graphicx}
\usepackage{bm, xcolor}
\usepackage{amssymb, amsmath}
\usepackage{multirow}
\usepackage{rotating}
\usepackage{array}
\usepackage{booktabs}
\usepackage{dcolumn}
\usepackage[numbers,sort&compress]{natbib}
\DeclareMathOperator*{\argmin}{arg\,min}
\DeclareMathOperator{\diag}{diag}
\DeclareMathOperator{\mseop}{MSE}
\DeclareMathOperator{\mean}{mean}
\newcolumntype{C}{>{\begin{sideways}}m{1.6ex}<{\end{sideways}}}
\newcolumntype{E}{>{\centering\arraybackslash}m{0.252\linewidth}}
\newcommand{\inp}[1]{\bm{r}\!\left(\bm{c}, #1\right)}

\newcommand{\mse}[1]{\mseop \left( #1,\bm{f} \right)}
\newcommand{\coloneqq}{\mathrel{\mathop:}=}
\newcommand{\rcoloneqq}{=\mathrel{\mathop:}}
\newcommand{\Real}{\mathbb{R}}
\makeatletter
\providecommand*{\diff}{\@ifnextchar^{\DIfF}{\DIfF^{}}}
\def\DIfF^#1{\mathop{\mathrm{\mathstrut d}}\nolimits^{#1}\gobblespace}%
\def\gobblespace{\futurelet\diffarg\opspace}%
\def\opspace{%
  \let\DiffSpace\!%
  \ifx\diffarg(%
  \let\DiffSpace\relax%
  \else%
  \ifx\diffarg\[%
  \let\DiffSpace\relax%
  \else%
  \ifx\diffarg\{%
  \let\DiffSpace\relax%
  \fi\fi\fi\DiffSpace}%
\makeatother
\begin{document}
\title{Optimising Spatial and Tonal Data for PDE-based 
Inpainting\thanks{Partly funded by the Deutsche Forschungsgemeinschaft 
(DFG).}}
\author{Laurent Hoeltgen\and Markus Mainberger\and Sebastian Hoffmann\and Joachim Weickert\and Ching Hoo Tang\and Simon Setzer\and Daniel Johannsen\and Frank Neumann\and Benjamin Doerr}
\authorrunning{Hoeltgen et al.}
\institute{
 L.\ Hoeltgen\\
  Faculty of Mathematics, Natural Sciences and Computer Science, 
  Brandenburg Technical University Cottbus--Senftenberg, 
  03046 Cottbus, Germany\\ 
  \email{hoeltgen@b-tu.de}
  \and
 M.\ Mainberger \and S.\ Hoffmann \and J.\ Weickert 
  \and S.\ Setzer\\
  Mathematical Image Analysis Group,
  Faculty of Mathematics and Computer Science, Campus E1.7,
  Saarland University, 66041 Saarbr\"ucken, Germany\\ 
  \email{\{mainberger, hoffmann, weickert, setzer\}\\
    @mia.uni-saarland.de}%
  \and
 C.\ H.\ Tang\\
  Research Group 1: Automation of Logic,
  Max Planck Institute for Informatics, Campus E1.4,
  66123 Saarbr\"ucken, Germany\\
  \email{chtang@mpi-inf.mpg.de}
  \and
 D.\ Johannsen\\
  School of Mathematical Sciences,
  Tel Aviv University, Ramat Aviv, Tel Aviv 69978, Israel\\
  \email{mail@danieljohannsen.net}
  \and
 F.\ Neumann\\
  School of Computer Science, Innova21 Building,
  University of Adelaide, Adelaide, SA 5005, Australia\\
  \email{frank@cs.adelaide.edu.au}
  \and
 B.\ Doerr\\
  LIX, 1 rue Honor\'e d'Estienne d'Orves, B\^{a}timent Alan Turing, 
  Campus de l'\'Ecole Polytechnique, CS35003, 91120 Palaiseau, 
  France\\
  \email{doerr@lix.polytechnique.fr}
}
\date{\today}
\maketitle
\cleardoublepage
% ---------------------------------------------------------------------------- %
% 
\begin{abstract}
 Some recent methods for lossy signal and image compression store only a 
 few selected pixels and fill in the missing structures by inpainting 
 with a partial differential equation (PDE). Suitable operators include 
 the Laplacian, the biharmonic operator, and edge-enhancing anisotropic 
 diffusion (EED). The quality of such approaches depends substantially 
 on the selection of the data that is kept. Optimising this data in the 
 domain and codomain gives rise to challenging mathematical problems
 that shall be addressed in our work.\par{}
 In the 1D case, we prove results that provide insights into the
 difficulty of this problem, and we give evidence that a splitting 
 into spatial and tonal (i.e.~function value) optimisation does hardly 
 deteriorate the results.
 In the 2D setting, we present generic algorithms that achieve a 
 high reconstruction quality even if the specified data is very sparse.
 To optimise the spatial data, we use a probabilistic sparsification,
 followed by a nonlocal pixel exchange that avoids getting trapped in 
 bad local optima. After this spatial optimisation we perform a tonal 
 optimisation that modifies the function values in order to reduce the
 global reconstruction error. For homogeneous diffusion inpainting, this 
 comes down to a least squares problem for which we prove that it has a 
 unique solution. We demonstrate that it can be found efficiently with 
 a gradient descent approach that is accelerated with fast explicit 
 diffusion (FED) cycles. Our framework allows to specify the desired
 density of the inpainting mask a priori. Moreover, is more generic 
 than other data optimisation approaches for the sparse inpainting 
 problem, since it can also be extended to nonlinear inpainting 
 operators such as EED. This is exploited to achieve reconstructions 
 with state-of-the-art quality.\par{}
 Apart from these specific contributions, we also give an extensive
 literature survey on PDE-based image compression methods.\par
 \vspace*{\baselineskip}
 \par
 \vspace*{\baselineskip}
\end{abstract}
 
%%%%%%%%%%%%%%%%%%%%%%%%%%%%%%%%%%%%%%%%%%%%%%%%%%%%%%%%%%%%%%%%%%%%%%%%%%%%%

\section{Introduction}
\label{sec:intro}

One of the most fascinating properties of variational methods and partial 
differential equations (PDE's) in image analysis is their property to 
fill in missing data. This filling-in effect has a long tradition:
It can be found already in the seminal optic flow paper of Horn and Schunck
\cite{HS81}, where the smoothness term propagates information to regions in
which the data term is small or even vanishes. More explicit interpolation 
ideas are exploited in so-called inpainting approaches. They go back to 
Masnou and Morel \cite{MM98a}, became popular by the work of 
Beltalm\'io {\em et al.} \cite{BSCB00}, and have been modified and
extended by many others such as \cite{BM07,CS01a,GS03,WW06}.
Inpainting problems arise when certain parts of the image domain 
are degraded in such a way that any useful data is unavailable. In 
the remaining parts of the image domain, the data is available and 
undegraded. The goal is to reconstruct the missing information by 
solving a suitable boundary value problem where the data in the 
undegraded regions serve as Dirichlet boundary conditions.

PDE-based image compression methods constitute a challenging application 
area of inpainting ideas; see Section \ref{sec: review} for a detailed 
review and many references.
These lossy compression approaches drive inpainting to the extreme: 
They store only a very small fraction of the data and inpaint the missing 
data with a suitable differential operator. However, PDE-based 
compression differs fundamentally from classical inpainting  
by the fact that it has the liberty to select the data that is 
kept. Typically one specifies only that a certain fraction of the 
image data is stored, 
and the codec\footnote{A codec is a system for image coding and decoding.} 
can optimise this data in order to minimise the reconstruction 
error. Obviously this involves a spatial optimisation step that selects
the locations of the kept pixels. We call these locations the {\em inpainting 
mask}. Interestingly, it may also be beneficial to optimise the grey values
of the selected pixels: While changing the grey values deteriorates the 
approximation quality inside the mask, it can improve the approximation 
in the (usually much larger) inpainting domain where no data is specified. 
This grey value optimisation is also called {\em tonal optimisation}.

In order to judge the potential of PDE-based compression approaches, it 
is important to go to the extreme and find the limits that inpainting 
with optimised sparse data can offer. This should include both spatial 
and tonal optimisation for different inpainting operators. 
In a first step, it can make sense to have a clear focus on quality and
postpone considerations of computational efficiency and coding costs 
to later work when the quality questions are answered in a satisfactory 
way.

Following this philosophy, the goal of our contribution is to explore 
the potential of PDE-based inpainting of sparse data that can be optimised 
both spatially and tonally. In order to evaluate different inpainting 
operators in a fair way, we aim at algorithms that are as generic and 
widely applicable as possible, and we optimise quality rather than 
speed. Computation time becomes only relevant for us when different 
methods lead to identical results. First our insights and 
algorithms are derived in more restricted settings. Afterwards, we 
investigate how these ideas can be extended and generalised.

Our work is organised as follows. We start by presenting a survey on
PDE-based image compression in Section~\ref{sec: review}.
Afterwards we discuss the homogeneous diffusion inpainting framework 
that is used for most of our optimisation algorithms in this work. In 
Section~\ref{sec:1-d-optimisation}, we present methods that aim at optimal
spatial and greyvalue data for homogeneous diffusion interpolation in 1D. 
The subsequent Section~\ref{sec:2-d-optimisation} deals with optimisation 
strategies in 2D. In Section~\ref{sec:extensions}, we present extensions 
to higher order and nonlinear diffusion inpainting operators. 
Finally we conclude with a summary in Section~\ref{sec:conclusion}.

Our discussions are based on a conference paper \cite{MHWT12}, in which we 
have introduced probabilistic sparsification and tonal optimisation for
homogeneous diffusion inpainting. Here we extend these results in a 
number of ways. Our main innovations are:
\begin{enumerate}
\item We present a new section that gives a detailed review of PDE-based
      image compression methods. Since this area has developped in a 
      very fruitful way and no review article is available so far, we feel
      that such a review section can be a useful starting point for
      readers who would like to learn more on this topic. 
\item In another new section we derive a mathematically sound approach 
      to solve the data optimisation problem for homogeneoues diffusion 
      inpainting of strictly convex signals in 1D.
      This restricted 1D setting allows us to gain insights into the 
      nonconvexity of the problem and to quantify the errors that are 
      caused by separating spatial and tonal optimisation.
\item We introduce a new algorithm for tonal optimisation, based on
      a gradient descent approach with an acceleration by a fast explicit
      diffusion (FED) strategy. While it achieves the same quality as 
      previous approaches, we show that it is more efficient. 
\item We explain how to extend our methods to more advanced linear or
      nonlinear inpainting operators such as biharmonic inpainting and 
      inpainting by edge-enhancing anisotropic diffusion. In 
      this way we achieve sparse PDE-based reconstructions of hitherto 
      unparalleled quality.
\end{enumerate}

%%%%%%%%%%%%%%%%%%%%%%%%%%%%%%%%%%%%%%%%%%%%%%%%%%%%%%%%%%%%%%%%%%%%%%%%%%

\section{A Review of PDE-Based Image Compression}
\label{sec: review}

Before describing our approach on spatial and tonal data optimisation
in the forthcoming sections, we first review existing approaches to 
PDE-based image compression. These methods belong to the class of 
lossy compression techniques. Hence, they aim at finding very 
compact file representations such that the resulting reconstructions
approximate the original image well.\par{}

PDE-based appoaches are alternatives to established transform-based 
codecs such as the JPEG standard~\cite{PM92} which uses the discrete 
cosine transform (DCT), or its successor JPEG 2000~\cite{TM02} that 
involves the wavelet transform.  However, since they follow a completely 
different concept, it is worthwhile to give an overview of their 
various aspects.\par{}

To this end, we review data optimisation questions, the choice of 
appropriate inpainting operators, and strategies to store the optimised 
data efficiently. Afterwards we describe PDE-based codecs that are 
based on specific image features, and we discuss algorithmic aspects, 
variants and extensions, as well as relations to other approaches.\par{}
 
Our review focuses on methods that use PDE's as a main tool
for compression. This means that we do not discuss the numerous PDE-based
or variational techniques that have been advocated as a preprocessing step
before coding images or videos (see e.g.~\cite{TSYK02,KS05}) or as a
postprocessing tool for removing coding artifacts
(such as \cite{Fo96,ADF05,BH12}).

%-----------------------------------------------------------------------

\subsection{Data Optimisation}

Gali\'c {\em et al.}~\cite{GWWB05,GWWB08} have pioneered research on
data selection methods for PDE-based image compression by proposing a
subdivision strategy that inserts mask points at locations where the
approximation error is too large. While these authors used a triangular 
subdivision, Schmaltz {\em et al.}~\cite{SPME14} modified this concept 
to rectangular subdivisions and added several algorithmic
improvements that allowed them to beat the quality of JPEG 2000 with
an anisotropic diffusion operator. Belhachmi {\em et al.}~\cite{BBBW08} 
have used the theory of shape optimisation to derive analytic results 
for the spatial optimisation problem in the case of homogeneous 
diffusion inpainting. Discrete approaches for the spatial and tonal 
optimisation problem with homogeneous diffusion inpainting have been 
presented by Mainberger {\em et al.}~\cite{MHWT12}. The spatial 
optimisation is based on a probabilistic sparsification strategy 
with nonlocal pixel exchange, while tonal optisation is formulated 
as a linear least squares problem. Hoeltgen {\em et al.}~\cite{HSW13} 
considered a control theoretic approach to the problem of data 
optimisation, also for homogeneous diffusion inpainting. They minimised 
the quadratic reconstruction error with a sparsity prior on the mask 
and a nonconvex inpainting constraint. Their numerical approach solves 
a series of convex optimisation problems with linear constraints. A 
similar constrained optimisation model was considered by Ochs {\em et 
al.}~\cite{OCBP14} who used this problem as a test scenario for their 
i{P}iano algorithm. Qualitatively both techniques gave similar results, 
but the i{P}iano approach offered advantages in terms of efficiency.

While all the before mentioned strategies pay much attention to spatial 
optimisation, less work has been devoted to tonal optimisation so far. 
Early heuristic attempts go back to Gali\'c {\em et al.}~\cite{GWWB08}. 
They lifted the grey values up or down to the next quantisation levels 
in order to improve the approximation quality in the vicinity of the 
data. A first systematic treatment in terms of a least squares 
optimisation problem was given by Mainberger {\em et al.}~\cite{MHWT12}, 
who used a randomised Gau{\ss}-Seidel algorithm. Faster numerical 
algorithms for the same problem have been proposed by Chen {\em et 
al.}~\cite{CRP14}, who applied the L-BFGS method, and by Hoeltgen and 
Weickert~\cite{HW15}, who advocated the LSQR algorithm and a 
primal--dual technique.

%-----------------------------------------------------------------------

\subsection{Finding Good Inpainting Operators}

Although many theoretical investigations on data selection methods are
based on homogeneous diffusion inpainting, the task of finding better
inpainting operators has been a research topic for a decade. Already
in 2005, Gali\'c {\em et al.} \cite{GWWB05} have shown that edge-enhancing
anisotropic diffusion (EED) \cite{We94e} is better suited for PDE-based 
image compression than homogeneous diffusion. The favourable performance
of anisotropic diffusion approaches of EED type has been confirmed by 
more detailed comparisons later on, both for randomly scattered data 
\cite{GWWB08,BU13} and for compression-optimised data \cite{CRP14,SPME14}.
Some of these evaluations involve many other inpainting operators such 
as biharmonic interpolation \cite{Du76} and its triharmonic counterpart, 
the absolute minimal Lipschitz extension \cite{CMS98}, isotropic 
nonlinear diffusion \cite{PM90,CLMC92} and its approximations of total 
variation (TV) inpainting \cite{CBAB97,ROF92}, as well as the inpainting 
operators of Bornemann and M\"arz \cite{BM07} and of Tschumperl\'e 
\cite{Ts06}. Total variation inpainting performed consistently bad, 
showing that operators which work fairly well for denoising are not 
necessarily good for sparse inpainting.
Bihamonic inpainting, on the other hand, turned out to be an interesting 
alternative to EED, when linearity or absence of any need for parameter 
specification are important, and over- and undershoots are acceptable.

%-------------------------------------------------------------------------

\subsection{Storing the Data}

Besides finding sparse inpainting data that allow to approximate the
original image in high quality, every practically useful codec must be 
able to store these data efficiently. Attempting this in a naive way would 
create a huge overhead, and the resulting codec would not be competitive 
to the quality that established transform-based codecs such as JPEG or 
JPEG 2000 can achieve for the same file size. The search for data that 
can be stored efficiently may even lead to compromises w.r.t. the data 
optimality: In order to use less bits, it is common to round the 
intensity values to a relatively small set of quantisation levels. 
Regarding the data localisation, it can be helpful to avoid 
a completely free configuration of data points and allow only masks 
which are represented efficiently with binary trees that result from 
subdivision strategies \cite{GWWB05,GWWB08,SPME14}.
Moreover, also lossless entropy coders such as Huffman coding 
\cite{Hu52}, JBIG \cite{Jb93}, or PAQ \cite{Ma09} are highly 
useful to remove the redundancy of the data.

%--------------------------------------------------------------------------

\subsection{Feature-based Methods}

Methods that rely on specific -- often perceptually relevant -- features 
rather than on data that are optimised for compression applications can 
be seen as predecessors or variants of PDE-based codecs. Let us now 
discuss the main types of such features and their relevance for coding.

%.......................................................................

\subsubsection{Contour-based Features}
Edges are perceptually relevant features that have been analysed
for a long time. Already in 1935, Werner \cite{We35} has investigated 
filling-in mechanisms from edges in biological vision systems. 
In computer vision, there are numerous appoaches in scale-space and 
wavelet theory that attempt to reconstruct an image from its edge 
information \cite{Lo77,COL85,ZR86,Ch87,HM89,GG90,MZ92a,Dr93,SH93a,El99},
often over multiple scales and supplemented with additional information.

Also within coding, publications that exploit information from edges or 
segment boundaries and combine it with inpainting processes have a long 
tradition \cite{Ca88,AG94,DMMH96,KGJ97,LSWL07,WZSG09,BHK10,MBWF11,ZD11,CJDL12,GLG12,LSJO12,HMWP13,GK14,BBG15}.
However, for general images these features are often suboptimal
as inpainting data. On the other hand, for specific applications
where the images are piecewise constant or piecewise smooth, such
codecs can achieve competitive results. This includes cartoon-like 
images \cite{MBWF11} or depth maps \cite{CJDL12,GLG12,HMWP13,LSJO12}.
For homogeneous diffusion inpainting and piecewise almost constant data, 
choosing data near edges can be justified by the theory of Belhachmi 
{\em et al.} \cite{BBBW08}. It suggests to select the mask density as 
an increasing function of the Laplacian magnitude.
Contour-based approaches also benefit from the fact that encoding of
contour data is relatively inexpensive: One can use e.g.~chain codes,
and a smooth intensity variation along a contour allows subsampling
without substantial quality degradations; see e.g.~\cite{MBWF11}.

Related ideas are also advocated in computer graphics for image editing
and vectorised image representations \cite{EG01,JCW11,OBWB08,OG11}.
Moreover, contour-based codecs that rely on inpainting processes with PDE's
have been used successfully for encoding digital elevation maps 
\cite{FLAP06,SCSA04,XFCA07}, where information is available in terms 
of level lines.

%.......................................................................

\subsubsection{Point-based Features}
In order to end up with more compact image representations, it appears
tempting to consider point-based features instead of contour-based ones.
Also here several results can be found in the literature.

Johansen {\em et al.} \cite{JSGA86} and Kanters {et al.} \cite{KLDJ05}
have performed research on reconstructing images from their top points 
in scale-space. Lillholm {\em et al.}~\cite{LNG03} present more general 
discussions on how to reconstruct an image from a suitable set of 
feature points and their derivatives (local jet).
More recently, also reconstruction attempts are described that 
are based on SIFT features \cite{WJP11} or local binary descriptors 
\cite{DJAV14}.

% Caselles {\em et al.} \cite{CCM96} argue that junction points are 
% perceptually relevant since they carry information about occlusions 
% and the spatial organisation of structures. Thus, they propose to 
% simplify an image by smoothing it with the affine morphological 
% scale-space, while preserving its junctions.

So far, all image reconstructions from isolated feature points that 
allow some perceptual interpretations or have shown their merits in 
other applications such as image matching do not offer a quality 
level that is sufficient for compression applications. It appears 
that these features are too sparse while lacking optimality.  

%-----------------------------------------------------------------------

\subsection{Fast Algorithms and Real-Time Aspects}

In real-time applications, the efficiency of a codec becomes a central
criterion. Here one should distinguish between real-time capabilities 
during encoding and real-time decoding. Often the latter one is more
important: For instance, when a company encodes a movie, it is acceptable 
if the creation of an optimised file takes many hours, as long as the 
customer is able to decode it in real-time.

On the encoding side, the approach of Belhachmi {\em et al.} 
\cite{BBBW08} is algorithmically sufficiently simple to allow real-time 
performance: 
It computes the Laplacian magnitude of a Gaussian-smoothed image and 
converts the result to a impainting binary mask by means of a halftoning 
method such as Floyd--Steinberg dithering \cite{FS76}. Also the codec of
Mainberger {\em et al.} \cite{MBWF11} is of comparable computational
simplicity. Most other approaches are not real-time capable during 
encoding since they apply a sophisticated optimisation of data and 
parameters. However, also for data optimisation, substantial algorithmic
accelerations have been achieved recently \cite{OCBP14,HW15}.

On the decoder side, no time-consuming optimisation steps are required,
and the main computational effort for PDE-based codecs consists of 
solving an inpainting problem. Thus, real-time decoding is possible 
if one uses appropriate algorithms and exploits the capabilities of 
modern hardware. In 2007, K\"ostler {\em et al.} \cite{KSFR07} presented
real-time algorithms for the subdivision approach of Gali\'c {\em et al.} 
\cite{GWWB05}. With multigrid methods for homogeneous diffusion inpainting
or lagged diffusivity multilevel approaches for anisotropic diffusion
inpainting, they could process more than 25 greyscale images per second 
of size $320 \times 240$ on a Sony PlayStation~3. Also Mainberger 
{\em et al.} \cite{MBWF11} proposed a multigrid approach for their 
edge-based homogeneous diffusion codec. On a PC CPU 
% (Intel Core 2 Duo T7500 at 2.20 GHz), 
they reported runtimes around $6$ frames per second for
decoding a $256 \times 256$ colour image. This is about $6$ times higher 
than JPEG 2000 and $24$ times higher than JPEG. For linear inpainting
problems with very sparse data, algorithms based on discrete Green's
functions can be an interesting alternative to multigrid methods 
\cite{HPW15}.
Recently, Peter {\em et al.} \cite{PSMM15} managed to decode $25$ 
greyscale images per second with VGA resolution $640 \times 480$ by 
means of anisotropic diffusion inpainting on a PC architecture with 
GPU support. 
% (Intel Xeon CPU 886 W3565@3.20GHz with Nvidia Geforce GTX 460). 
This was accomplished with so-called FED schemes \cite{GWB10} 
that are well-suited for parallel implementations.
 
%---------------------------------------------------------------------------

\subsection{Hybrid Image Compression Methods}

The usefulness of embedding inpainting concepts into existing 
transform-based image compression methods is studied in a number of
publications \cite{LSWL07,RSB03,WSWX06,XSWL07,YS12}. Their goal 
is to benefit simultaneously from the advantages of inpainting methods 
and transform-based codecs. While PDE-based inpainting approaches often
perform better at edges, transform-based codecs are favoured inside the 
regions, in particular if they are dominated by texture. By construction, 
such hybrid methods restrict themselves to the main application fields 
of both paradigms rather than pushing inpainting ideas to the limit.

An example of a hybrid approach that remains entirely within the 
inpainting framework is \cite{PW15}. It combines EED-based inpainting 
with the sparse, exemplar-based appoach of Facciolo {\em et al.} 
\cite{FACS09}. This results in a more faithful recovery of textured 
regions than a pure EED-based method.

Chan and Zhou \cite{CZ00} have proposed a hybrid method that modifies 
wavelet-based image compression by a TV regularisation of the wavelet 
coefficients. More recently, Moinard {\em et al.} \cite{MABD11} 
have used TV inpainting to predict DCT coefficients in a JPEG-based codec.

%---------------------------------------------------------------------------

\subsection{Modifications, Extensions, and Applications}

A progressive mode option is a useful feature of a codec. It allows 
to encode and transmit data in a coarse-to-fine way. Thus, in the 
decoding phase, the representation can be subsequently refined 
while reading the data stream. In \cite{SMMW13}, two progressive modes 
have been suggested for the EED-based codec, and it has been shown 
that for high compression rates, their quality is competitive to JPEG 
and JPEG 2000.

By construction, PDE-based codecs are well-suited for encoding
specific regions of interest with higher precision. All one has to
do is to increase the density of the inpainting data in the region
of interest. For more details, we refer to \cite{PSMM15}.

Since a large fraction of modern imagery consists of colour image data, 
codecs should be able to handle colour images efficiently. 
While for most PDE approaches, extensions to vector- and even 
matrix-valued data exist, these modifications are not necessarily optimal 
for compression applications. As a remedy, in \cite{PW14} an extension of 
the EED-based codec to colour images is proposed that uses the YCbCr 
representation. The perceptually relevant luma channel is stored with 
fairly high accuracy, while the chroma channels are encoded with very 
sparse data. In the reconstruction phase the chroma inpainting is 
guided by the structural information from the luma channel.

Extending PDE-based codecs from 2-D images to three-dimensional data sets
does not create severe difficulties. Most PDE's have natural extensions
to higher dimensions, and also other concepts generalise in a natural
way: For instance a rectangular subdivision in 2D is replaced by a cuboidal 
subdivision in 3D \cite{SPME14}. Because of the richer neighbourhood
structure, the redundancy in higher dimensional data allows to achieve a 
higher compression rate for the same quality as a lower dimensional codec.

PDE-based compression strategies have also been investigated for encoding 
contours \cite{SPME14} and surfaces \cite{BW08,ROR13}. These publications 
follow the philosophy of PDE-based image compression, but replace the 
Laplacian by the contour curvature or the Laplace--Beltrami operator.

Obviously one can apply any method for compressing 3D data also to 
video coding, if one models a video or parts of it as a spatiotemporal 
data block. Other video coding approaches have been suggested that use 
inpainting concepts within the H.264/MPEG-4 AVC video coding standard 
\cite{DNLK10,MABD11,LSW12}.
In \cite{SW12}, the authors have implemented a video compression system 
that combines PDE-based image compression with pose tracking.

Cryptographic applications are studied in \cite{MSBW12}, where 
the authors combine PDE-based compression with steganographic 
concepts in order to hide one image in another one, or parts of 
an image in other parts.

%---------------------------------------------------------------------------

\subsection{Relations to Other Methods}

PDE-based codecs that select the inpainting mask with triangular 
or rectangular subdivisions have structural similarities to piecewise 
polynomial image
approximations based on adaptive triangulations or quadtrees. Extensive
research has been performed on these approximations; see 
\cite{DLR90,St91,SB94,DNV97,DDI06,KDN07,LUH07,Ko07,BPC09,SA09,CDHM12,Li14}
and the references therein.
Often such approaches use linear polynomials within each triangle or
rectangle. In this case, one may also interpret them as a solution of
the Laplace equation with Dirichlet boundary data obtained from a linear
interpolation of the vertex data. Thus, they can be seen as specific
localised PDE-based codecs. Alternative interpretations can be given in
terms of finite element approximations. In the one-dimensional case, the
linear spline approximation is even fully equivalent to homogeneous
diffusion inpainting.

Inpainting with linear differential operators allows also an analytical
representation of its solution in terms of the Green's function of
the operator. This has been used in \cite{HPW15} to relate linear
PDE-based inpainting to sparsity concepts: Discrete Green's functions
serve as atoms in a dictionary that gives a sparse representation
of the inpainting solution. In the one-dimensional case, discrete
analytic derivations are presented in \cite{PHW15}.

On the other hand, continuous Green's functions are also used as
radial basis functions in scattered data interpolation \cite{Bu03}.
Moreover, some radial basis functions can be seen as rotationally
invariant multidimensional extensions of spline interpolation.
This establishes a connection between PDE-based inpainting and
image representations in terms of radial basis functions and
splines, such as \cite{ASHU05,US05,FP10,DFTT11}.

%%%%%%%%%%%%%%%%%%%%%%%%%%%%%%%%%%%%%%%%%%%%%%%%%%%%%%%%%%%%%%%%%%%%%%%%%%%%%

\section{Inpainting with Homogeneous Diffusion}
\label{sec:inpainting}

In this section, we will briefly present the homogeneous diffusion 
reconstruction method (also known as Laplace interpolation) that 
lies at the basis of our approach and will 
be used for most of the results presented here. Homogeneous diffusion 
is among the simplest inpainting processes that one can consider. This 
makes it suitable for theoretical investigations. Nevertheless, one 
should emphasise that even such a simple method can yield very good 
results if the interpolation data is chosen in an appropriate way; 
see e.g.\ \cite{BBBW08,HSW13,MHWT12,OCBP14}.

Let $f:\Omega\to\Real$ be a smooth function on some bounded domain
$\Omega\subset\Real^n$ with a sufficiently regular boundary $\partial\Omega$.
Throughout this work, we will restrict ourselves to the case $n=1$ (1D signals)
and $n=2$ (greyscale images), although many results will also be valid for
arbitrary $n\geqslant 1$. Moreover, let us assume that there exists a set of
known data $\Omega_K\subsetneqq\Omega$. Homogeneous diffusion inpainting
considers the following partial differential equation with mixed boundary
conditions.
\begin{align}
 \Delta u = 0  \quad 
   &\text{on}\ \Omega\setminus\Omega_K, \label{eq:icd1}\\
 u = f  \quad 
   &\text{on}\ \Omega_K, \label{eq:icd2}\\
 \partial_n u = 0  \quad
   &\text{on}\ \partial\Omega\setminus{}\partial\Omega_K,
                    \label{eq:icd3}
\end{align}
where $\partial_n u$ denotes the derivative of $u$ in outer normal direction. 
We assume that both boundary sets $\partial\Omega_K$ and
$\partial\Omega\setminus{}\partial\Omega_K$ are non-empty. Equations of this
type are commonly referred to as mixed boundary value problems and sometimes 
also as Zaremba's problem named after Stanislaw Zaremba, who studied such
equations already in 1910 \cite{Z1910}. The existence and uniqueness of
solutions has been extensively analysed during the last century. Showing that
\eqref{eq:icd1}--\eqref{eq:icd3} is indeed solvable is by no means trivial. 
Generally, one can either show the existence of solutions in very weak settings
or one has to impose strong regularity conditions on the domain. The references
\cite{AK1982,M1970} discuss the solvability in a general way. In \cite{M1955} 
it is shown that a H\"older continuous solution exists if the data is 
sufficiently regular.
In \cite{B1994}, the author discusses the regularity of solutions on Lipschitz
domains. A more general existence theory for solutions is given in \cite{F1949}.
Further investigations on mixed boundary value problems can also be found in
\cite{GT2001,LU1968}. A particularly easy case is the 1D setting, where the
solution can obviously be expressed using piecewise linear splines 
interpolating data on $\partial\Omega_K$.\par{}
For convenience we introduce the confidence function $c$ which states whether a
point is known or not:
\begin{equation}
  c\left(x\right) \coloneqq 
  \begin{cases}
    1 &\text{for}\ x\in\Omega_K,\\
    0 &\text{for}\ x\in\Omega\setminus\Omega_K.
  \end{cases}
\end{equation}
Then it is possible to write \eqref{eq:icd1}--\eqref{eq:icd3} as
\begin{align}
 & c(x)\,(u(x)\!-\!f(x)) - (1\!-\!c(x))\,\Delta u(x) = 0 
 \quad \text{on}\ \Omega, \label{eq:ci1}\\
 &\partial_n u = 0 \quad \text{on}\ \partial\Omega\setminus\partial\Omega_K.
 \label{eq:ci2}
\end{align}
For most parts of this text we will prefer this formulation, as it is more
comfortable to work with in the discrete setting, which can be obtained as
follows. Let $J \coloneqq \{1,\dots,N\}$ be the set of indices enumerating the
discrete sample positions, and $K\subseteq J$ the subset of indices of known
samples. Then we can express the discrete version of $f$ as a vector
$\bm{f} = \left(f_1,\dots,f_N\right)^\top$ and the corresponding solution as a
vector $\bm{u} \in \mathbb{R}^N$. The binary mask $\bm{c} \in \mathbb{R}^N$, 
where $c_{i}\coloneqq 1$ if $i\in K$ and $c_{i}\coloneqq 0$ otherwise, 
indicates the positions of the Dirichlet boundary data. At last, the 
Laplacian $\Delta$ is discretised by standard means of finite 
differences~\cite{MM05}. Hence, a
straightforward discretisation of \eqref{eq:ci1}--\eqref{eq:ci2} on a 
regular grid yields
\begin{equation}
  \label{eq:discinp}
  \bm{C}\left(\bm{u}-\bm{f}\right) -
  \left(\bm{I} - \bm{C}\right) \bm{A}\bm{u} = \bm{0}
\end{equation}
where $\bm{I}$ is the identity matrix,
$\bm{C} \coloneqq \diag\left(\bm{c}\right)$ is a diagonal matrix with the
components of $\bm{c}$ as its entries, and $\bm{A}=(a_{i,j})$ is a 
symmetric $N\times N$ matrix, describing the discrete Laplace operator 
$\Delta$ with homogeneous Neumann boundary conditions on 
$\partial\Omega\setminus{}\partial\Omega_K$. 
Its entries are given by
\begin{equation}
  a_{i,j} = \left\{
   \begin{array}{r@{\qquad}l}
   \dfrac{1}{h_\ell^2} & (j \in \mathcal{N}_\ell(i)),\\[0.4cm]
   -\displaystyle \sum_{\ell\in \{x,y\}} 
    \displaystyle \sum_{j\in \mathcal{N}_\ell(i)} 
   \dfrac{1}{h_\ell^2}& (j=i),\\[0.4cm]
    0 & (\textnormal{else})\;,
  \end{array}\right.
\end{equation}
where $\mathcal{N}_\ell(i)$ are the neighbours of pixel $i$ in 
$\ell$-direction, and $h_\ell$ is the corresponding grid size.
By a simple reordering of the terms, \eqref{eq:discinp} can be rewritten 
as the following linear system:
\begin{equation}
  \label{eq:discinp2}
  \underbrace{\left(\bm{C} - \left(\bm{I} - \bm{C}\right)
      \bm{A}\right)}_{\rcoloneqq\bm{M}}\bm{u} = \bm{C}\bm{f}.
\end{equation}
It has been shown in \cite{MBWF11} that this linear system of equations has 
a unique solution and that it can be solved efficiently with bidirectional 
multigrid methods.

%%%%%%%%%%%%%%%%%%%%%%%%%%%%%%%%%%%%%%%%%%%%%%%%%%%%%%%%%%%%%%%%%%%%%%%%%%%%%%

\section{Optimisation Strategies in 1D}
\label{sec:1-d-optimisation}

The choice of the position of the Dirichlet boundary data has a strong
influence on the quality of the reconstruction. In order to gain some 
analytical insight on this problem, we restrict ourselves in this section 
to the 1D continuous setting and consider in 
Section~\ref{sec:optim-knots-interp} how to find the optimal positions 
of the mask points $\bm{c}$. Optimising position and value of the 
Dirichlet data simultaneously will be treated in 
Section~\ref{sec:optim-knots-appr}.
% In the spline literature the 1D data optimisation problem is known under
% the name {\em free knot problem} \cite{B1974a,BR1968a,DT2008,J1978,NB1982}.
% Our considerations will also make use of results obtained by the spline
% approximation community.

%---------------------------------------------------------------------------- 

\subsection{Optimal Knots for Interpolating Convex Functions}
\label{sec:optim-knots-interp}

In this section, we assume that $f:\left[a,b\right]\to\Real$ is always a 
strictly convex and continuously differentiable function. Our goal is to 
find a distribution of $N+1$ knot sites $\lbrace c_i\rbrace_{i=0}^{N}$ in 
the interval $\left[a,b\right]$ such that the interpolation error with 
piecewise linear splines becomes minimal in the $L_{1}$ norm. For 1D 
formulations, this is equivalent to determining the Dirichlet boundary data 
in \eqref{eq:icd1}--\eqref{eq:icd3} such that the solution $u$ gives the 
best possible reconstruction to $f$ in the $L_{1}$ sense. In concrete 
terms, this means that we seek $N+1$ positions inside the interval 
$\left[a,b\right]$ with
\begin{equation}
  \label{eq:FreeKnotDistribution}
  a \rcoloneqq c_0 < c_1 < c_2 < \ldots < c_{N-1} < c_{N} \coloneqq b
\end{equation}
and a piecewise linear spline $L\left(x;\lbrace c_i\rbrace_{i=0}^{N}\right)$
interpolating $f$ at the positions $c_i$, such that the error
\begin{equation}
  \label{eq:FreeKnotEnergy}
  E\left(\lbrace c_i\rbrace_{i=0}^{N}\right) \coloneqq
  \int_a^b \left| L\left(x;\lbrace c_i\rbrace_{i=0}^{N}\right) -
    f(x) \right|\diff{x}
\end{equation}
becomes minimal. This optimisation problem is also called {\em free knot
problem} and has been studied for more than fifty years. We refer to
\cite{H2002,KS1978} for similar considerations as in our work and to
\cite{B1974a,DP1987,DT2008,J1978} and the references therein for more general
approaches. Note that it is quite common to relax the interpolation condition
and to generalise the problem to explicitly allow approximating functions.
Further details on interpolation and approximation techniques can be found in
\cite{R1964}. Alternative ways to optimise linear spline interpolation are
discussed e.g.~in \cite{BTU2004}.

For technical reasons, the knots $c_{0}$ and $c_{N}$ in \eqref{eq:FreeKnotDistribution} are fixed at the boundary of the considered interval. The choice of the
$L_{1}$ norm is especially attractive in this case: Due to the convexity of $f$,
the integrand
\begin{equation*}
  L\left(x; \lbrace c_i\rbrace_{i=0}^{N}\right) - f(x)
\end{equation*}
is nonnegative for all $x$ in $\left[ a,b \right]$. Thus, we can simply omit
the absolute value in \eqref{eq:FreeKnotEnergy}. This observation simplifies 
the derivation of optimality conditions below. Furthermore, the 
requirement that the linear spline $L$ must coincide with $f$ at the knot 
sites $c_{i}$ allows us to state the interpolating function in an analytic 
form: If $x\in\left[c_{i},c_{i+1}\right]$, then
\begin{equation}
  L\left(x;\lbrace c_i\rbrace_{i=0}^{N}\right) =
  \frac{f(c_{i+1})-f(c_i)}{c_{i+1}-c_i}(x-c_i) + f(c_i).
\end{equation}
Note that requiring that the $c_i$ are distinct is necessary to avoid
a division by $0$. A straightforward computation gives 
\begin{align}
  \label{eq:FreeKnotError}
  \nonumber E\left(\lbrace c_i\rbrace_{i=0}^{N}\right)
  &= \int_a^b \left( L\left(x;\lbrace c_i\rbrace_{i=0}^{N}\right)
    - f(x)\right) \, \diff{x} \\ \nonumber
  &= \frac{1}{2} \sum_{i=0}^{N-1} \left(c_{i+1}-c_i\right)
    \left(f\left(c_{i+1}\right)+f\left(c_i\right)\right) \\
  &\qquad-\int_a^b f\left(x\right)\diff{x}.
\end{align} 
This expression also corresponds to the error of the composite trapezoidal 
rule for the numerical integration of $f$ with non-equidistant integration 
intervals.

Equation~\eqref{eq:FreeKnotError} will be our starting point for developing an
algorithm to determine the optimal knot sites. However, before we will do so, 
we present a result which shows the difficulty of the free knot optimisation 
problem.
\begin{proposition}\label{th:ErrorNonConvex}
  If the function $f:\left[a,b\right]\rightarrow\Real$ is strictly convex and
  twice continuously differentiable, then the energy \eqref{eq:FreeKnotEnergy} 
  is convex in $\lbrace c_i\rbrace_{i=0}^{N}$ for 3
  knots (i.e.\ $N=2$). In general, it is not convex for any other number of 
  knots larger than 3 (i.e.\ $N>2$).
\end{proposition}

\begin{proof}
 In the case of three knots, we only have one free variable, and it follows from
 \eqref{eq:FreeKnotError} that the error is given by
 \begin{multline}
   E\;(c_{1}) = \frac{1}{2}\biggl( \left(c_{1}-a\right)
   \bigl(f\left(a\right) + f\left(c_{1}\right)\bigr)\biggr.\\
   + \biggl.\left(b-c_{1}\right)
   \bigl(f\left(b\right)+f\left(c_{1}\right)\bigr) \biggr) -
   \int_a^b f\left(x\right)\diff{x}.
 \end{multline}
 Further, the second derivative of $E$ is given by
 \begin{equation}
   \frac{\partial^2}{\partial c_{1}^2}\;E\;(c_{1}) =
   \frac{b-a}{2}\frac{\partial^2}{\partial c_{1}^2}f(c_{1}) > 0
   \quad{}\forall c_{1}\in\left[a,b\right].
 \end{equation}
 In order to demonstrate that the error function is nonconvex for a higher
 number of interpolation points, it suffices to provide a counterexample. Let
 us consider the function $f(x)=\exp\left(x\right)$ on the interval
 $\left[-15,15\right]$ as well as the two knot sets
 $U_1 = \lbrace -15, 10.65, 14.65, 15\rbrace$ and
 $U_2 = \lbrace -15, -1.2, 12.5, 15\rbrace$. If $E$ were convex, then it must
 also be convex along the line in $\Real^4$ that connects $U_{1}$ and $U_{2}$
 (interpreting both knot sets as vectors in $\Real^{4}$). However, the plot
 of $E((1-t)U_{1}+tU_{2})$ with $t\in\left[0,1\right]$ depicted in
 Fig.~\ref{fig:NonConvexEnergy} displays a nonconvex behaviour.
 \qed{}
\end{proof}

%...........................................................................

\begin{figure}[hbtp]
\centering
\includegraphics[width=8cm]{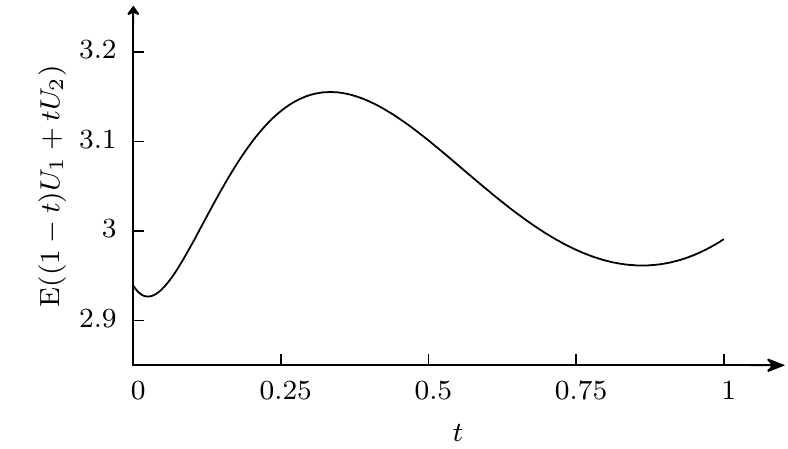}
\caption{Evolution of the error for interpolating the function $\exp(x)$
  with knots along the segment with bounds
  $U_1= ( -15, 10.65, 14.65, 15 )^{\top}$ and
  $U_2= ( -15, -1.2, 12.5, 15 )^{\top}$. It shows a nonconvex
  behaviour. For better readability, the function values have been 
  rescaled by a factor $10^{-6}$.}
\label{fig:NonConvexEnergy}
\end{figure}

%...........................................................................

We remark that the previous proposition does not claim that the energy can 
never be convex for more than three knots. Indeed, for affine functions of 
the form $mx+k$ with real coefficients $m$ and $k$ the energy is identical 
0 for any number of knots, and thus also convex. This shows that even under 
weaker conditions as in the proposition, the energy may be convex.

%...........................................................................
 \pagebreak
\subsubsection{A New Algorithm for the Free Knot Problem for Linear Spline
  Interpolation}\nopagebreak
\label{sec:energy-minim-free}

A necessary condition for a minimiser $\lbrace c_{i}^{*}\rbrace_{i=0}^{N}$ of
\eqref{eq:FreeKnotError} is
$\bm{\nabla} E\left(\lbrace c_{i}^{*}\rbrace_{i=0}^{N}\right)=0$. However, 
it follows from Proposition~\ref{th:ErrorNonConvex} that this condition is 
not sufficient. There may exist several global and/or local minima. A
simple computation leads to the following system of ${N\!-\!1}$ nonlinear 
equations in the ${N\!-\!1}$ unknowns $c_1$,...,$c_{N-1}$:
\begin{equation}
  \label{eq:FreeKnotEqSystem}
  f'\left(c_i\right) = \frac{f\left(c_{i+1}\right) -
    f\left(c_{i-1}\right)}{c_{i+1}-c_{i-1}},\quad i=1,\ldots,N-1.
\end{equation}
It should be noted that each knot only depends on its direct
neighbours. Therefore, odd indexed knots only depend on even indexed knots and
vice versa. Since $f$ is strictly convex, it follows that $f'$ is strictly
monotonically increasing. Thus, its inverse exists and is unique at every point
of the considered interval. This motivates the following iterative scheme.

%.........................................................................
\begin{samepage}
  \begin{description}
    \item[\textbf{Algorithm 1:}] \textbf{Spatial Optimisation in 1D}\hfill
          \hrule\vspace{2mm}
    \item[\emph{Input:}]\hfill\\
          $N+1$, the desired number of knots.
    \item[\emph{Initialisation:}]\hfill\\
          Choose any initial distribution $\{ c_{i}^{0}\}_{i=0}^{N}$ with $c_{0}^{0}=a$
          and $c_{N}^{0}=b$, e.g.\ a uniform distribution of the knots on the interval
          $\left[a,b\right]$.
    \item[\emph{Repeat until a fixed point is reached:}] \hfill{}
          \begin{description}
            \item[Update even knots for all possible $i$:]\hfill
                  \begin{equation}
                    c_{2i}^{k+1} \coloneqq \left(f^{\prime}\right)^{-1}
                    \left( \frac{f\left(c_{2i+1}^{k}\right) -
                        f\left(c_{2i-1}^{k}\right)}{c_{2i+1}^{k}-c_{2i-1}^{k}} \right).
                    \label{eq:FreeKnotItera}
                  \end{equation}
            \item[Update odd knots for all possible $i$:]\hfill
                  \begin{equation}
                    c_{2i+1}^{k+1} \coloneqq \left(f^{\prime}\right)^{-1}
                    \left( \frac{f\left(c_{2i+2}^{k+1}\right) -
                        f\left(c_{2i}^{k+1}\right)}{c_{2i+2}^{k+1}-c_{2i}^{k+1}} \right).
                    \label{eq:FreeKnotIterb}
                  \end{equation}
          \end{description}
    \item[\emph{Output:}]\hfill\\
          The final knot distribution $\{ c_i^*\}_{i=0}^{N}$.\\
          \vspace*{-2mm}\hrule\vspace*{\baselineskip}
  \end{description}
\end{samepage}
%.........................................................................
 
Observe that the above scheme is similar to a Red-Black Gau{\ss}--Seidel 
scheme for the solution of linear systems: We update the variables 
iteratively and use newly gained information as soon as it becomes 
available without interfering with the direct neighbours of the data point. 
An important issue is that the knots $c_{i}$ are not allowed to fall 
together. The following proposition shows that this cannot happen.

%............................................................................
 
\begin{proposition}
  The iterative scheme (\ref{eq:FreeKnotItera})--(\ref{eq:FreeKnotIterb}) 
  pre\-serves the ordering of the knot positions. Thus, we have e.g.
  \begin{equation}
    c_{i-1}^{k} < c_{i}^{k} < c_{i+1}^{k}
    \:\Rightarrow\: c_{i-1}^{k+1} < c_{i}^{k+1} < c_{i+1}^{k+1}
    \quad \forall\ k,i.
  \end{equation}
\end{proposition}

%............................................................................

\begin{proof}
 Since $f$ is differentiable on $\left[c_{i-1},c_{i+1}\right]$ for all $i$,
 the mean value theorem guarantees the existence of a $c_{i}$ in
 $\left(c_{i-1},c_{i+1}\right)$ such that
 \begin{equation}
   f'\left(c_i\right) =
   \frac{f\left(c_{i+1}\right)-f\left(c_{i-1}\right)}{c_{i+1}-c_{i-1}}\,.
 \end{equation}
 Thus, our iterative scheme must necessarily preserve the order of the
 knots.\qed{}
\end{proof}

%............................................................................

The next theorem shows that the iterates from our algorithm monotonically
decrease the considered energy.
\begin{theorem}
  If the function $f:\left[a,b\right]\rightarrow\Real$ is strictly convex and
  twice continuously differentiable, the iterates
  $\left( \lbrace c_{i}^{k} \rbrace_{i=0}^{N}\right)_{k}$ obtained in
  \eqref{eq:FreeKnotItera} and \eqref{eq:FreeKnotIterb} decrease the $L_1$ error
  \eqref{eq:FreeKnotEnergy} in each step, i.e.\ we have
  \begin{equation}
    E\left(\lbrace c_{i}^{k+1} \rbrace_{i=0}^{N}\right)
    \leqslant E\left(\lbrace c_{i}^{k} \rbrace_{i=0}^{N}\right)\ \forall k.
  \end{equation}
\end{theorem}
\begin{proof}
 By alternating between the update of the odd and even indexed sites, 
 the problem decouples.
 The new value of $c^{k+1}_{i}$ only depends on $c^k_{i-1}$ and
 $c^k_{i+1}$, which are fixed. Therefore, the problem is localised, and we 
 can update all the even/odd indexed knots independently of each other. It
 follows that one iteration step is equivalent to finding the optimal
 $c^{k+1}_{i}$ such that the interpolation error becomes minimal on
 $\left[c^{k}_{i-1},c^{k}_{i+1}\right]$ for all even/odd $i$. The global
 error can now be written as the sum of all the errors over the intervals
 $\left[c^{k}_{i-1},c^{k}_{i+1}\right]$ and will necessarily decrease when
 each term of this sum decreases. Proposition~\ref{th:ErrorNonConvex} shows
 that the considered energy is convex for three knots. Thus,
 $\bm{\nabla} E(\lbrace c^{k}_{j} \rbrace_{j=i-1}^{i+1}) = 0$ is not only a
 necessary, but also a sufficient condition for being a minimum on
 $\left[c^{k}_{i-1},c^{k}_{i+1}\right]$. This means that
 \eqref{eq:FreeKnotItera} will not increase the error when updating even
 indexed knots and subsequently, \eqref{eq:FreeKnotIterb} will not increase
 the error while updating the odd numbered sites. Therefore, it follows that
 the overall error cannot increase in an iteration step.\qed{}
\end{proof}
Since the error is bounded from below by 0, we also obtain the following result.
\begin{corollary}
  The sequence $\left(E\left(\lbrace c_{i}^{k} \rbrace_{i=0}^{N}\right)\right)_k$
  is convergent.
\end{corollary}
Note that the previous statements do not claim the convergence of the sequence
$\left(\lbrace c_{i}^{k} \rbrace_{i=0}^{N}\right)_{k}$. Since the problem is 
nonconvex, the global minimum of the considered energy is not necessarily 
unique.
In that case, our algorithm might alternate between several of the minimisers.
These minimisers are, from a qualitative point of view, all equivalent, since 
they yield the same error. However, they might not be the 
global minimiser. Also
note that the theorem of Bolzano-Weierstrass asserts that
$\left(\lbrace c_{i}^{k} \rbrace_{i=0}^{N}\right)_{k}$ contains at least one
convergent subsequence since all the $c_{i}^{k}$ must necessarily lie in the
interval $\left[a,b\right]$. Finally, we remark that in our test cases the
results were always of very good quality. This gives rise to the conjecture 
that the found knot distributions are close to a global minimum.
 
%............................................................................

\subsubsection{Numerical Experiments}
\label{sec:numer-exper-Interp}

Let us now perform experiments with our new algorithm. We consider the convex
function
\begin{equation*}
  f\left(x\right)=\exp{\left(2x-3\right)}+x
\end{equation*}
on the interval $\left[-4,4\right]$. Figure~\ref{fig:ErrorP1} depicts the
evolution of the error, while Fig.~\ref{fig:PositionP1} exhibits the resulting
distribution of the knots. The experiments were done with a randomised initial
distribution of the knots and $5000$ iterations. Interestingly, the iterates
always converged already after very few iterations. In accordance with the
theory from the previous section, we also note that the error is monotonically
decreasing, both with respect to the number of knots and with the number of
iterations. Another interesting observation is the influence of an additionally
introduced knot on the other knots: Adding further interpolation sites has a
global influence on all the other knots. Moreover, we encounter a higher knot
density in regions with large curvature than in flat regions. This is in
agreement with the theory of Belhachmi {\em et al.}~\cite{BBBW08}.
 
%...........................................................................

\begin{figure}[hbtp]
  \centering
  \includegraphics[width=0.7\linewidth]{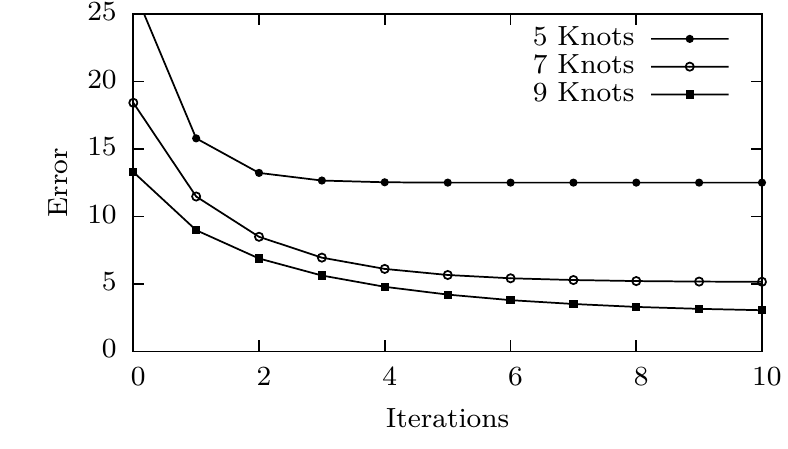}
  \caption{Evolution of the $L_1$ error as a function of the iterates for
    different numbers of knots for the function
    $f\left(x\right)=\exp{\left(2x-3\right)}+x$ on the interval
    $\left[-4,4\right]$. Note that the error is decreasing both with respect to
    the number of knots and with respect to the number of iterations.
    Furthermore, the curves suggest a rather fast convergence to
    a stationary energy value.}
  \label{fig:ErrorP1}
\end{figure}

%...........................................................................
 
\begin{figure}
  \centering
  \includegraphics[width=0.7\linewidth]{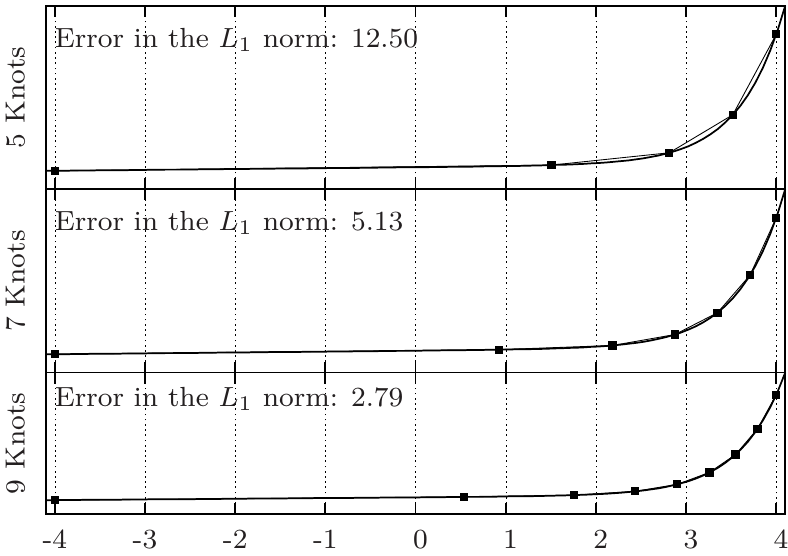}
  \caption{Distribution of the knots corresponding to Figure~\ref{fig:ErrorP1}
    for the converged state of the function $f(x)=\exp(2x-3)+x$. Note that there
    are almost no knots in flat regions, whereas there is a high density in
    regions with large curvature. Changing the number of knots actually
    influences the position of all the knots.}
  \label{fig:PositionP1}
\end{figure}

%---------------------------------------------------------------------------- 

\subsection{Optimal Knots for Approximating Convex Functions}
\label{sec:optim-knots-appr}

In the previous section, we have seen how to optimise the {\em location} of 
the interpolation data. A next step would be to investigate how much an 
optimisation of the {\em grey value data} at these sites could further 
improve the result. To do so, we no longer require that the function value 
and the value of the interpolant must coincide at the knot locations. Instead, 
we consider an approximation problem and require that the overall 
reconstruction minimises the $L_{1}$ error on the considered domain. As in 
the previous section, we again use piecewise linear splines. Such 
approximations by means of first degree splines have a long history, and 
there exist results for many special cases. In \cite{S1961}, the best 
approximation of strictly convex functions in the least squares sense has
been analysed, while \cite{D1963} cites general conditions for the determination
of best approximations of strictly convex functions in the $L_{\infty}$ sense.
Theoretical results can also be found in \cite{J1978}. In \cite{NB1982}, it is
shown that an optimal approximation of convex functions with splines is not
necessarily unique, a problem which we already mentioned in the stricter case of
convex spline interpolation. Finally, \cite{C1971,P1968} supply algorithms for
determining such approximations.

In order to minimise the $L_{1}$ error between a strictly convex function
$f:\left[a,b\right]\to\Real$ and a piecewise linear spline in an approximation
setting, we use an algorithm by Hamideh \cite{H2002}. It relies on
a classical result from approximation theory:
\begin{theorem}
  \label{th:BestLocalPoints}
  For any function $f\in C\left(\left[a,b\right]\right)$, which is strictly
  convex in $\left[a,b\right]$, the optimal straight line approximation to $f$
  in the $L_{1}$ sense on $\left[a,b\right]$ interpolates the function at the
  points
  \begin{equation}
    \xi_1 \coloneqq \frac{3}{4}a+\frac{1}{4}b\quad\text{and}\quad\xi_2
    \coloneqq \frac{1}{4}a+\frac{3}{4}b\,.
  \end{equation}
\end{theorem}
\begin{proof}
 This result can be verified by direct computation. Alternatively, a
 detailed proof can also be found in \cite{R1964}.
 \qed
\end{proof}
This gives rise to the following algorithm.
  \begin{description}
    \item[\textbf{Algorithm 2:}]\textbf{Optimal Approximation in 1D}\hfill
          \hrule\vspace{2mm}
    \item[\emph{Input:}]\hfill\\
          $N+1$ the number of desired knots.
    \item[\emph{Initialisation:}]\hfill\\
          Choose an arbitrary distribution of $N+1$ knots with $c_0\coloneqq a$ and
          $c_{N}\coloneqq b$.
    \item[\emph{Compute:}]\hfill\\
          Repeat these steps until a fixed point is reached.
          \begin{enumerate}
            \item On each subinterval $\left[c_{i-1},c_{i}\right]$, define the points
                  \begin{equation}
                    \label{eq:LocallyOptPoints}
                    \xi_{i,1} \coloneqq  \frac{3c_{i-1}+c_{i}}{4},\qquad \xi_{i,2}
                    \coloneqq  \frac{c_{i-1}+3c_{i}}{4}
                  \end{equation}
                  as well as the corresponding line $\ell_{i-1}$ passing through these 
                  points:
                  \begin{equation}
                    \ell_{i-1}\left(x\right) \coloneqq \frac{f\left(\xi_{i,2}\right) -
                      f\left(\xi_{i,1}\right)}{\xi_{i,2}-\xi_{i,1}}\left(x-\xi_{i,1}\right) +
                    f\left(\xi_{i,1}\right)
                  \end{equation}
            \item Determine for all $i$ the new knot position $c_i$ by intersecting the
                  lines $\ell_{i-1}$ and $\ell_{i}$, e.g.\ solve
                  \begin{equation}
                    \ell_{i-1}\left(c_i\right) = \ell_{i}\left(c_i\right)
                  \end{equation}
                  for $c_{i}$.
          \end{enumerate}
    \item[\emph{Output:}]\hfill\\
          The final knot distribution $\{c_i\}_{i=0}^N$.\\
          \vspace*{-2mm}\hrule\vspace*{\baselineskip}
  \end{description}
The algorithm of Hamideh is similar to our Algorithm~1 for interpolation: 
Both use iteratively locally optimal solutions to perform a global 
optimisation. In addition, the following properties are shown in 
\cite{H2002}:
\begin{enumerate}
\item The resulting sequence of $L_{1}$ errors
$\left(E(\{c_i^{k}\}_{i=0}^N)\right)_{k}$ is convergent.
\item For all $i=1,\ldots,N-1$, it holds that
  \begin{subequations}
    \begin{align}
      \liminf_{k\rightarrow\infty} &\left|c_{i+1}^{k} - c_{i}^k\right| > 0, \\
      \lim_{k\rightarrow\infty} &\left|c_{i}^{k+1}-c_i^{k}\right|=0,
    \end{align}
  \end{subequations}
  where $\left(\lbrace c_{i}^{k}\rbrace_{i=0}^{N}\right)_{k}$ is our sequence of
  knot sets. Thus, two distinct knots cannot fall together.
\item If the unknown optimal spline fulfils certain continuity conditions, one
  can guarantee that the above algorithm convergences towards an optimal
  solution.
\end{enumerate}
Further analytic results on the optimal knot distribution can also be found in
\cite{KS1978}.

%............................................................................

\subsubsection{Numerical Experiments}
\label{sec:numer-exper-Approx}

We investigate the knot distribution for different number of knots and repeat
the experiments of Section~\ref{sec:numer-exper-Interp} using now the algorithm
of Hamideh. Our main interest lies in the performance of a combined optimisation
over a sequential optimisation of the position and the corresponding value of
the interpolation data. To this end, we also consider a tonal optimisation 
that adjusts the function values for our free knot result. This additional 
task is carried out as a postprocessing step. Since the position of the mask 
values is optimised in an $L_{1}$ setting, we use the same framework for the 
tonal optimisation, too. Due to the fact that for a fixed set of knots, we 
can express our interpolating spline as a linear combination of first 
degree B-Splines, we can express the tonal optimisation as a linear 
regression task with respect to the $L_{1}$ norm. It is well-known that 
such problems can be reduced to linear programs which can efficiently be 
solved by standard solvers from the literature. Our findings are 
summarised in Tab.~\ref{tab:ErrorsFreeKnot}, and a visual comparison 
between the obtained mask sets for a set of seven knots is given in 
Fig.~\ref{fig:HamidehReconst}.

%...........................................................................
 
\begin{table}
\centering
\caption[Errors for optimal mask interpolation and approximation]
  {Error measures for our interpolation algorithm and the approximation 
   algorithm of Hamideh for different numbers of mask points applied to 
   the function $x\mapsto\exp(2x-3)+x$ on the interval $\left[-4,4 \right]$. 
   For our method we list the error without additional tonal optimisation 
   and with additional tonal optimisation.}
\begin{tabular*}{0.6\linewidth}{@{\extracolsep{\fill}}lrrr@{\extracolsep{\fill}}}
 \toprule\addlinespace
 \multirow{2}{*}[-1mm]{\centering$N+1$} & 
 \multicolumn{2}{c}{Our method} & 
 \multirow{2}{*}[-1mm]{Hamideh} \\ 
   \cmidrule{2-3} & 
   \multicolumn{1}{c}{no tonal optim.} & 
   \multicolumn{1}{c}{with tonal optim.} &  \\ 
 \addlinespace \cmidrule{1-1} 
 \cmidrule{2-2} 
 \cmidrule{3-3} 
 \cmidrule{4-4} 
 \addlinespace
    5    & 12.501   & 4.229   & 3.982 \\
    7    &  5.134   & 1.810   & 1.748 \\
    9    &  2.785   & 0.999   & 0.977 \\ 
 \addlinespace \bottomrule
\end{tabular*}
\label{tab:ErrorsFreeKnot}
\end{table}

%............................................................................
 
Although the knot distribution in Fig.~\ref{fig:HamidehReconst} for the approach
of Hamideh is similar to the one found with interpolation, the corresponding
errors in Tab.~\ref{tab:ErrorsFreeKnot} are significantly lower, since the
reconstruction can adapt much better to the original function when compared to
the interpolation framework. Nevertheless, we also observe a substantial gain
of the tonal optimisation. Even though we cannot outperform the combined
optimisation, we achieve competitive results, in particular for larger values
of $N$.\par 

%............................................................................

\begin{figure}[btp]
\centering
\includegraphics[width=0.9\linewidth]{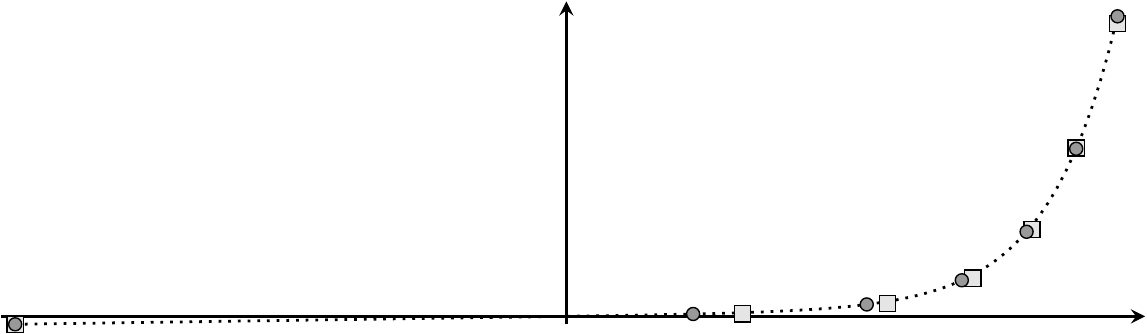}
\caption{
  Comparison between the knots found with our method (dark grey disks)
  and the method of Hamideh (light grey squares) for the function
  $\exp(2x-3)+x$ along the interval $[-4,4]$ (dotted line).}
\label{fig:HamidehReconst}
\end{figure}

%............................................................................
 
Our investigations in the 1D case show that a careful optimisation of the
positions (spatial information) and greyvalues (tonal information) of the data
can lead to significant improvements compared to a purely spatial tuning.
Moreover, a sequential optimisation is almost as good as a combined 
optimisation of the data. The next step in our strategy will be to adapt 
these ideas to the two-dimensional setting such that we can efficiently 
apply them on discrete image data.

%%%%%%%%%%%%%%%%%%%%%%%%%%%%%%%%%%%%%%%%%%%%%%%%%%%%%%%%%%%%%%%%%%%%%%%%%%%%%

\section{Optimisation Strategies in 2D}
\label{sec:2-d-optimisation}
 
Unfortunately, our interpolation algorithm from
Section~\ref{sec:1-d-optimisation} can hardly be used directly on 2D image data.
First of all, we would be restricted to convex/concave images. Of course, one
could always segment an arbitrary image into convex and concave regions and
treat them separately, but in many cases this would lead to heavily
oversegmented images and a suboptimal global distribution of the data points
for the reconstruction. Secondly, the solution of 
\eqref{eq:ci1}--\eqref{eq:ci2} in higher dimensions cannot be written 
in terms of simple piecewise linear interpolation: To characterise it
analytically would require more complicated expressions that involve
Green's functions \cite{HPW15}. 
Therefore, we want to consider other approaches here. Nevertheless, they 
exploit the basic ideas and findings of a spatial and tonal optimisation 
from the previous section.

For practical reasons, we present a two-step optimisation strategy: First
we consider an interpolation approach to optimise the spatial data.
Afterwards we optimise the tonal information at the points obtained in the
first step. This strategy can be justified by the fact that in the 1D case, 
the obtained knot distributions for the interpolation and approximation 
algorithms were similar. For the optimisation of the data sites we investigate 
two methods: An analytic approach proposed in \cite{BBBW08} that exploits 
the theory of shape optimisation, and our probabilistic sparsification 
approach from \cite{MHWT12}. Finally we complement our pure interpolation 
framework with a best approximation scheme that incorporates tonal 
optimisation in our model.

%---------------------------------------------------------------------------- 
\pagebreak
\subsection{Optimising Spatial Data}\nopagebreak
\label{sec:optim-spat-data}

%............................................................................

\subsubsection{Analytic Approach}
\label{sec:analytic-approach}

In order to approach the question about the optimal data selection, Belhachmi 
{\em et al.}~\cite{BBBW08} use the mathematical theory of shape 
optimisation. This theory optimises topological properties of given objects. 
Belhachmi {\em et al.}~seek the optimal shape of the set of Dirichlet 
boundary data in \eqref{eq:icd1}--\eqref{eq:icd3}. 
They show that the density of the data points
should be chosen as an increasing function of the Laplacian magnitude of the
original image. However, this optimality result returns a continuous density
function rather than a discrete pixel mask. This yields an additional problem
that is also discussed in \cite{BBBW08}, namely, how to obtain the best
discrete (binary) approximation to a continuous density function. Belhachmi 
{\em et al.}~suggest the following strategy to obtain a point mask based upon 
$|\Delta f|$.
First one applies a small amount of Gaussian presmoothing with standard
deviation $\sigma$ to obtain $f_\sigma$. This is a common procedure in image
analysis to ensure the differentiability of the data. Then one computes the
Laplacian magnitude $|\Delta f_\sigma|$ and rescales it such that its mean
represents the desired point density given as fraction $d$ of all pixels.
Finally any dithering algorithm that preserves the average grey value can be
applied to obtain the binary point mask.

In \cite{BBBW08} the classical error diffusion method of Floyd and
Steinberg~\cite{FS76} is used. However, we favour the more sophisticated
electrostatic halftoning~\cite{SGBW10} over simpler dithering approaches, since
it has proven to yield very good results for discretising a continuous
distribution function.
Figure~\ref{fig:cmp_dither} shows the superiority of
electrostatic halftoning over Floyd-Steinberg dithering in terms of the mean 
squared error (MSE)
\begin{equation}
  \mse{\bm{u}} = \frac{1}{|J|}\sum_{i \in J} (f_i - u_i)^2,
  \label{eq:defmse}
\end{equation}
where $\bm{u}$ denotes the reconstruction, $\bm{f}$ the original image and 
$J$ the set of all pixel indices. Since any dithering method introduces 
errors, it remains an open question if this is the most suitable approach 
to discretise the continuous optimality result.

%..........................................................................

\begin{figure*}
  \centering
  \begin{minipage}{0.235\textwidth}
    \centering
    \includegraphics[width=\textwidth]{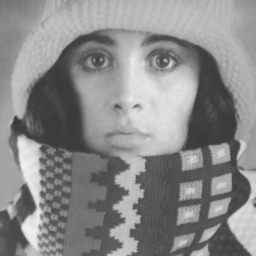}\\
    \textbf{(a)}
  \end{minipage}\hfill
  \begin{minipage}{0.235\textwidth}
    \centering
    \includegraphics[width=\textwidth]{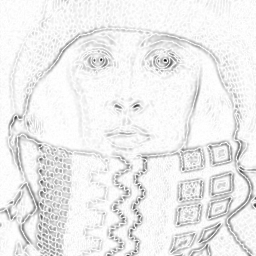}\\
    \textbf{(b)}
  \end{minipage}\hfill
  \begin{minipage}
    {0.235\textwidth}
    \centering
    \includegraphics[width=\textwidth]{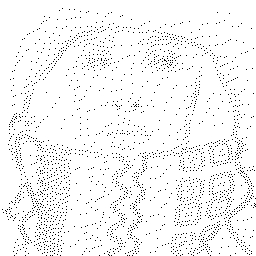}\\
    \textbf{(c)}
  \end{minipage}\hfill
  \begin{minipage}{0.235\textwidth}
    \centering
    \includegraphics[width=\textwidth]{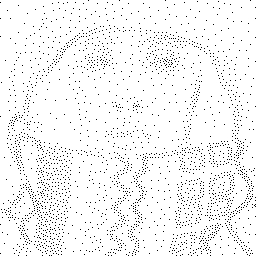}\\
    \textbf{(d)}
  \end{minipage}
  \caption{ 
  \textbf{(a)} Original test image \emph{trui} ($256\times256$ pixels). 
  \textbf{(b)} Smoothed Laplacian magnitude of (a) using $\sigma=1$
               (rescaled and inverted). 
  \textbf{(c, d)} Dithered versions of (b) using Floyd-Steinberg error 
               diffusion and electrostatic halftoning, respectively.
               With (c) and (d) as mask for homogeneous diffusion inpainting 
               we obtain an MSE of $138.98$ and $101.14$, respectively. 
  All the images use floating point values in the range from 0 to 255 for 
  the pixels. The discrete masks have a density of 4\%.}
  \label{fig:cmp_dither}
\end{figure*}

%..........................................................................

The theory of Belhachmi {\em et al.}~\cite{BBBW08} demands the data 
points to be chosen as an increasing function of $|\Delta f|$. However, 
the optimal
increasing function depends on the details of the underlying model. One option 
is to use the identity function of the Laplacian magnitude, as was done in 
\cite{MHWT12}. In the present work, we introduce an additional parameter 
$s>0$ and dither $|\Delta f_{\sigma}|^{s}$ instead. 
This choice can also be motivated from the original paper \cite{BBBW08} and 
allows to tune the density of the selected points in homogeneous regions. 
The complete method, which we call the analytic approach, is summarised below.
\pagebreak
\begin{samepage}
  \begin{description}
    \item[\textbf{Algorithm 3:}]\textbf{Analytic Approach}\hfill
          \hrule\vspace{2mm}\nopagebreak
    \item[\emph{Input:}]\hfill\\
          Original image $f$, Gaussian standard deviation $\sigma$, exponent $s$, desired pixel density $d$.
    \item[\emph{Compute:}]\hfill
          \begin{enumerate}
            \item Perform Gaussian presmoothing with standard deviation $\sigma$: $f_\sigma = K_\sigma \ast f$.
            \item Compute $|\Delta f_\sigma|^s$.
            \item Rescale $|\Delta f_\sigma|^s$ to
                  \begin{equation*}
                    \frac{d\cdot f_{\max}}{\mean(|\Delta f_\sigma|^s)}\cdot |\Delta f_\sigma|^s
                  \end{equation*}
                  where $f_{\max}$ is the maximal possible grey value.
            \item Apply electrostatic halftoning to obtain $\bm{c}$.
          \end{enumerate}
    \item[\emph{Output:}]\hfill\\
          Discrete pixel mask $\bm{c}$.\\
          \vspace*{-2mm}\hrule
  \end{description}
\end{samepage}
%............................................................................

\begin{figure*}
\centering
 \begin{minipage}{0.235\textwidth}
   \centering
   \includegraphics[width=\textwidth]{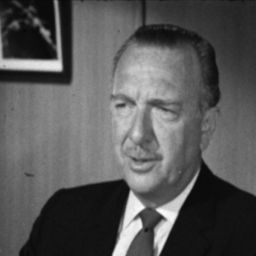}\\
   \textbf{(a)}
 \end{minipage}\hfill
 \begin{minipage}{0.235\textwidth}
   \centering
   \includegraphics[width=\textwidth]{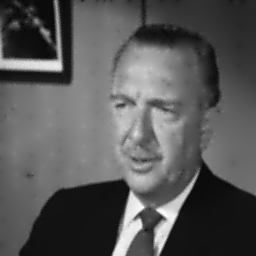}\\
   \textbf{(b)}
 \end{minipage}\hfill
 \begin{minipage}{0.235\textwidth}
   \centering
   \includegraphics[width=\textwidth]{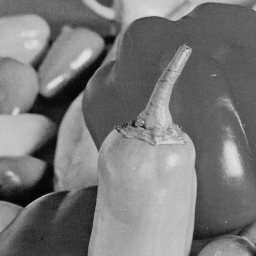}\\
   \textbf{(c)}
 \end{minipage}\hfill
 \begin{minipage}{0.235\textwidth}
   \centering
   \includegraphics[width=\textwidth]{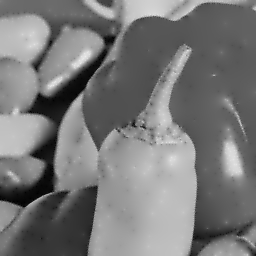}\\
   \textbf{(d)}
 \end{minipage}
\caption{
 \textbf{(a, c)} Original test images \emph{walter} ($256\times256$
   pixels) and \emph{peppers} ($256\times256$ pixels). 
 \textbf{(b, d)} Best reconstruction results with 4\% of all pixels, 
   using probabilistic sparsification, nonlocal pixel exchange and grey 
   value optimisation. The image (b) has an MSE of 12.45, while image (d) 
   has an MSE of 25.10.}
\label{fig:resultswalterpeppers}
\end{figure*}

%............................................................................

In order to evaluate the analytic approach, let us apply it on the test image
\emph{trui} (see Fig.~\ref{fig:cmp_dither}(a)). We aim at a mask pixel density
of 4\% ($d=0.04$) of all pixels. The parameters $\sigma$ and $s$ are chosen 
such that the MSE of the reconstruction becomes minimal. This was achieved with 
$\sigma=1.6$ and $s=0.8$. For comparison, we also consider two masks with the 
same amount of pixels: a mask with points on a regular grid and a randomly 
sampled mask. Figure~\ref{fig:results} shows these two masks and the one
created with the analytic approach in the first column as well as the
corresponding reconstructions in the second column. We observe that the
reconstruction quality highly benefits from a dedicated point selection. These
results are confirmed by the Columns 3, 4 and 5 of Tab.~\ref{tab:results}, 
which depicts quantitative results for several test images from 
Fig.~\ref{fig:cmp_dither}(a), Fig.~\ref{fig:resultswalterpeppers}(a), 
and Fig.~\ref{fig:resultswalterpeppers}(c).

As already mentioned, the analytic approach is real-time capable if a fast 
dithering method is used. However, rather than on speed, the focus 
of our present work is to maximise the quality of the reconstruction. 
Therefore, in the next subsection we present an alternative algorithm that 
is slower but outperforms the analytic approach in terms of quality.
 
%............................................................................

\subsubsection{Probabilistic Sparsification}
\label{sec:prob-spars}
 
We have seen that the analytic approach offers a clean strategy how to 
choose optimal spatial data in a continuous image. However, due to certain 
degrees of freedom in the modelling, errors caused by the dithering 
algorithm, and discretisation effects, its results on digital images 
cannot be optimal.
As an alternative, we consider now a discretise--then--optimise strategy. 
This way we can directly search for a binary-valued mask $\bm{c}$ by
working with the discrete inpainting formulation from \eqref{eq:discinp}. An
immediate consequence of this approach is that there are only finitely many
combinations for $\bm{c}$. Unfortunately, already for an image of size
$256{\times}256$ pixels and a desired pixel density of 4\% there are
$\binom{65536}{2621} \approx 2 \cdot 10^{4777}$ possible masks. To tackle
this combinatorial problem we suggest a method called probabilistic
sparsification. It uses a greedy strategy to reduce the search space.

Let $\bm{f}$ be a given discrete image and let $\inp{\bm{f}}$ be the function
that computes the solution $\bm{u}$ of the discrete homogeneous inpainting
process \eqref{eq:discinp2} with a mask $\bm{c}$:
\begin{equation}
  \inp{\bm{f}} \coloneqq \bm{u} =
  \left(\bm{C}-(\bm{I}-\bm{C})\bm{A}\right)^{-1} \bm{C} \bm{f}.
  \label{eq:inpainting_function}
\end{equation}
The goal is to find the pixel mask $\bm{c}$ that selects a given fraction
$d$ of all pixels and minimises $\mse{\bm{u}}$.

Starting with a full mask, where every pixel is chosen, probabilistic
sparsification iteratively removes the least significant mask pixels until a
desired density is reached. More specifically, we randomly choose a fraction $p$
of candidate pixels from the current mask. These pixels are removed from the
mask, and an inpainting reconstruction is calculated. The significance of a
candidate pixel can then be estimated by computing the local error, i.e.\ the
squared grey value difference of the inpainted and original image in this pixel.
Then we permanently remove the fraction $q$ of the candidates that exhibit the
smallest local error from the mask, and we insert back again the remaining 
fraction $(1-q)$ of the candidates. A detailed description of our algorithm 
is given below.

%..........................................................................

\begin{samepage}
  \begin{description}
    \item[\textbf{Algorithm 4:}] \textbf{Stochastic Sparsification}\hfill
          \hrule\vspace{2mm}
    \item[\emph{Input:}]\hfill\\
          Original image $\bm{f}$, fraction $p$ of mask pixels used as candidates,
          fraction $q$ of candidate pixels that are removed in each iteration, desired
          pixel density $d$.
    \item[\emph{Initialisation:}]\hfill\\
          $\bm{C} \coloneqq \diag\left( 1,\dots,1 \right)^\top$ and $K\coloneqq J$.
    \item[\emph{Compute:}]\hfill\\
          Do 
          \begin{enumerate}
            \item Choose randomly a candidate set $T$ of $p\cdot |K|$ pixel indices 
                  from $K$.
            \item For all $i \in T$ set $c_i \coloneqq 0$.
            \item Compute $\bm{u} \coloneqq \inp{\bm{f}}$.
            \item For all $i\in{}T$ compute the error $e_i=(u_i-f_i)^2$.
            \item For all $i$ of the $(1-q) \cdot |T|$ largest values of
                  $\{e_i\ |\ i \in T\}$ reassign $c_i\coloneqq 1$.
            \item Remove the indices $i\not\in T$ from $K$ and clear $T$.
          \end{enumerate}
          while $|K| > d \cdot |J|$.
    \item[\emph{Output:}]\hfill\\
          Pixel mask $\bm{c}$, such that $\sum_{i\in J} c_i = d \cdot |J|$\\
          \vspace*{-2mm}\hrule
  \end{description}
\end{samepage}
%..........................................................................

The larger the parameters $p$ and $q$ are chosen, the faster the algorithm
converges, since in each step, $p\cdot q \cdot |K|$ pixels are removed. After 
$k$ steps there are $(1-pq)^k\cdot|J|$ mask pixels left. Hence, for a 
density $d$, the algorithm terminates after at most
$\left\lceil\frac{\ln{d}}{\ln{(1-pq)}}\right\rceil$ iterations, where
$\lceil\cdot\rceil$ denotes the ceiling function, giving the smallest 
integer not less than its argument.\par
Because there is a global interdependence between all selected mask pixels,
probabilistic sparsification cannot guarantee to give optimal solutions.
Therefore, the question arises how the parameters $p$ and $q$ influence 
the quality of the resulting mask. To this end, we run several experiments
with different $p$ and $q$. The results are depicted in Tab.~\ref{tab:pqtest}.
Note that we set the candidate set as well as the set of pixels that are 
removed to $1$ if $p$ or $q$ would lead to sets smaller than one pixel.
The optimal $p$ is usually not very large: If too many candidates are 
removed from the mask, the local error does not provide enough information 
to select good pixels to remove permanently. The parameter $q$ can usually 
be chosen as small as possible, i.e.\ such that only one
candidate is removed in each iteration. For our test images, this was the
case for $q=10^{-6}$. With larger values for $q$, the probability that we 
remove important pixels increases.
 
%............................................................................

\begin{table*}
\centering
\caption{
 Influence of the parameters $p$ and $q$ of probabilistic sparsification. 
 There have been in total $100$ runs for each pair $(p, q)$ on the test 
 image \emph{trui} with desired pixel density $d=0.04$. Numbers in the 
 table are the mean and standard deviation of the MSE. Again all pixel 
 values lie in the interval $\left[0,255\right]$.}
\resizebox{\textwidth}{!}{
\begin{tabular}{rccccccc}\toprule
 \multirow{2}{*}{$q$\hspace*{5mm}} & \multicolumn{7}{c}{$p$}\\ 
 \cmidrule(lr){2-8} & 0.01 & 0.02 & 0.05 & 0.1 & 0.2 & 0.3 & 0.4 \\ 
 \cmidrule(r){1-1} 
 \cmidrule(lr){2-2} 
 \cmidrule(lr){3-3} 
 \cmidrule(lr){4-4} 
 \cmidrule(lr){5-5} 
 \cmidrule(lr){6-6} 
 \cmidrule(lr){7-7} 
 \cmidrule(lr){8-8}
 \addlinespace
 $10^{-6}$ & 
   103.7 $\pm$ 1.88 & 
   \phantom{0}98.2 $\pm$ 1.72 & 
   \phantom{0}87.9 $\pm$ 1.61 & 
   77.6 $\pm$ 1.40 & 
   67.7 $\pm$ 1.40 & 
   \textbf{66.1 $\pm$ 1.36} & 
   70.7 $\pm$ 1.76 \\
 $10^{-3}$ & 
   103.7 $\pm$ 1.77 & 
   \phantom{0}98.0 $\pm$ 1.87 & 
   \phantom{0}87.7 $\pm$ 1.74 & 
   77.4 $\pm$ 1.33 & 
   67.8 $\pm$ 1.26 & 
   66.3 $\pm$ 1.41 & 
   70.5 $\pm$ 1.67 \\
 $10^{-2}$ & 
   103.1 $\pm$ 1.69 & 
   \phantom{0}98.3 $\pm$ 1.91 & 
   \phantom{0}88.9 $\pm$ 1.76 & 
   81.8 $\pm$ 1.68 & 
   73.0 $\pm$ 1.44 & 
   68.9 $\pm$ 1.81 & 
   69.4 $\pm$ 1.84 \\
 $2\cdot 10^{-2}$ & 
   103.7 $\pm$ 1.99 & 
   \phantom{0}98.7 $\pm$ 1.74 & 
   \phantom{0}92.6 $\pm$ 1.57 & 
   85.3 $\pm$ 1.71 & 
   76.4 $\pm$ 1.78 & 
   71.4 $\pm$ 1.56 & 
   70.6 $\pm$ 1.83 \\
 $5\cdot 10^{-2}$ & 
   104.9 $\pm$ 2.19 & 
   102.7 $\pm$ 1.96 & 
   \phantom{0}98.1 $\pm$ 2.07 & 
   91.6 $\pm$ 1.82 & 
   82.7 $\pm$ 2.02 & 
   77.1 $\pm$ 2.03 & 
   74.6 $\pm$ 1.79 \\
 $10^{-1}$ & 
   110.3 $\pm$ 2.36 & 
   107.2 $\pm$ 2.63 & 
   103.7 $\pm$ 2.26 & 
   97.7 $\pm$ 2.14 & 
   89.5 $\pm$ 2.00 & 
   83.9 $\pm$ 2.22 & 
   80.8 $\pm$ 2.24 
 \\\addlinespace
 \bottomrule
\end{tabular}}
\label{tab:pqtest}
\end{table*}

%............................................................................

Tab.~\ref{tab:pqtest} illustrates the robustness of the algorithm. 
Note that the approach is not deterministic and thus always
returns different results since the candidates are chosen randomly at 
each iteration. Although the obtained masks for different seeds differ 
in most of the selected pixels, we obtain qualitatively comparable 
results:
The standard deviation does not exceed a value of $2.6$ and is even smaller 
for optimal values for $p$ and $q$, when the algorithm is run several 
times.

In Fig.~\ref{fig:results}, the images in the third row show the results of the
probabilistic sparsification for the test image \emph{trui} with a mask pixel
density of 4\% ($d=0.04$). To optimise the quality, we use $p=0.3$ and 
$q=10^{-6}$ (cf.\ Tab.~\ref{tab:pqtest}). Both the visual as well as the
quantitative results outperform the ones of the analytic approach. 
This can also be observed for the two test images \emph{walter} and 
\emph{peppers}, as is shown in Tab.~\ref{tab:results}. The parameters 
$p$ and $q$ are optimised for each image individually.

%...........................................................................

\subsubsection{Nonlocal Pixel Exchange}
\label{sec:non-local-pixel}

As we have seen in the previous section, probabilistic sparsification
outperforms the analytic approach. Nevertheless, it is not guaranteed to find an
optimal solution. An obvious drawback of probabilistic sparsification is the
fact that due to its greedy nature, once a point is removed, it will never be
put back into the mask again. Thus, especially at later stages, where only few
mask pixels are left, important points might have been removed, so that we 
end up in a suboptimal local minimum. We now present a method called nonlocal 
pixel exchange that allows to further improve the results of any previously 
obtained mask, in our case the one from probabilistic sparsification. It 
starts with a sparse, possibly suboptimal mask that contains already the 
desired density $d$ of mask pixels. In each step, it randomly selects a 
set of $m$ non-mask pixels as candidates. The candidate that exhibits the 
largest local error is then exchanged with a randomly chosen mask pixel. 
If the inpainting result with the new mask is worse than before, we revert 
the switch. Otherwise we proceed with the new mask.
By construction, the nonlocal pixel exchange can only improve the result. 
This algorithm always converges towards an optimal solution in terms of a
exchange of two pixels. Since we exchange at each iteration the same 
number of candidate pixels it follows that this approach is not 
equivalent to an exhaustive search through all possible combinations. 
Thus, one cannot guarantee convergence towards the global minimum.
The description below shows the details of the algorithm.

%...........................................................................

\begin{samepage}
  \begin{description}
    \item[\textbf{Algorithm 5:}] \textbf{Nonlocal Pixel Exchange}\hfill
          \hrule\vspace{2mm}
    \item[\emph{Input:}]\hfill\\
          Original image $\bm{f}$, pixel mask $\bm{c}$, size $m$ of candidate set, 
          the set $K$ of pixel indices of the mask $\bm{c}$.
    \item[\emph{Initialisation:}]\hfill\\
          $\bm{u} \coloneqq \inp{\bm{f}}$ and $\bm{c}^{\text{new}} \coloneqq \bm{c}$.
    \item[\emph{Compute:}]\hfill\\
          Repeat
          \begin{enumerate}
            \item Choose randomly $m \leq |K|$ pixel indices $i$ from $J \setminus K$ 
                  and compute the local error $e_i \coloneqq (u_i - f_i)^2$.
            \item Exchange step:\\
                  Choose randomly a $j \in K$ and set $c^{\text{new}}_j \coloneqq 0$.\\
                  For the largest value of $e_i$, set $c^{\text{new}}_i\coloneqq 1$.
            \item Compute $\bm{u}^{\text{new}} \coloneqq r(\bm{c}^{\text{new}}, \bm{f})$.
            \item If $\mse{\bm{u}} > \mse{\bm{u}^{\text{new}}}$
                  \begin{itemize}
                    \item[] $\bm{u} \coloneqq \bm{u}^{\text{new}}$ and $\bm{c} \coloneqq
                          \bm{c}^{\text{new}}$.
                    \item[] Update $K$.
                  \end{itemize}
                  else
                  \begin{itemize}
                    \item[] Reset $\bm{c}^{\text{new}} \coloneqq \bm{c}$.
                  \end{itemize}
          \end{enumerate}
          until no pairs can be found for exchange.
    \item[\emph{Output:}]\hfill\\
          Optimised mask $\bm{c}^{\text{new}}$.\\
          \vspace*{-2mm}\hrule
  \end{description}
\end{samepage}
%...........................................................................

As for the previous algorithms, we are interested in an optimal parameter
selection. Table~\ref{tab:mntest} shows the results for different choices 
of $m$ when the nonlocal pixel exchange is applied to the masks from the 
previous section (randomly selected, regular grid, and analytic approach from
Fig.~\ref{fig:results}) and in addition the one we obtained by probabilistic
sparsification (also Fig.~\ref{fig:results}).
Technically one could always choose $m=1$, but a noticeable speedup can be
obtained by choosing a larger $m$. In our experiments, values around 
$m=20$ result in fastest convergence.
 
%............................................................................

\begin{table*}
 \centering
 \caption{
   Mean squared error after $500{,}000$ iterations for different values
   for $m$, when the nonlocal pixel exchange is applied to different masks 
   of the test image {\em trui} (see also Fig.~\ref{fig:results}). The best 
   result for each mask is marked in boldface.}
 \begin{tabular*}{0.9\linewidth}{@{\extracolsep{\stretch{1}}}rcccccccc@{}}\toprule
   \multirow{2}{*}[-1mm]{mask\hspace{1cm}}        
   & \multicolumn{8}{c}{$m$}\\ 
   \cmidrule{2-9} & 1 & 5 & 10 & 20 & 30 & 40 & 50 & 100\\ 
   \cmidrule{1-1} 
   \cmidrule{2-2} 
   \cmidrule{3-3} 
   \cmidrule{4-4} 
   \cmidrule{5-5} 
   \cmidrule{6-6} 
   \cmidrule{7-7} 
   \cmidrule{8-8} 
   \cmidrule{9-9}
   randomly selected            
   & 49.88 & 46.43 & 44.99 & \textbf{44.78} & 45.00 & 45.07 & 45.25 & 48.22\\
   regular grid                 
   & 49.67 & 45.79 & 45.19 & \textbf{44.76} & 45.16 & 45.56 & 45.76 & 48.33\\
   analytic approach            
   & 49.19 & 45.88 & 45.14 & \textbf{44.70} & 45.33 & 45.56 & 46.13 & 49.31\\
   probabilistic sparsification 
   & 43.72 & 42.45 & 42.34 & \textbf{41.92} & 41.97 & 42.49 & 42.19 & 43.65\\
   \addlinespace\bottomrule
 \end{tabular*}
 \label{tab:mntest}
\end{table*}
 
%............................................................................

The nonlocal pixel exchange improves the masks from any method we have
considered so far; see Fig.~\ref{fig:nlpe-test-masks}. Especially
within the first few iterations we achieve significant quality gains. After
$500{,}000$ iterations, we reach with all three masks an MSE below 
$45$. The best reconstruction with an MSE of $41.92$ is obtained 
with the mask from probabilistic sparsification. It is depicted in 
Fig.~\ref{fig:results}. As for the previous methods, Tab.~\ref{tab:results} 
provides also quantitative results for the test images \emph{walter} and 
\emph{peppers}. They support the above observations.

%...........................................................................
 
\begin{figure}
  \centering
  \includegraphics[width=\linewidth]{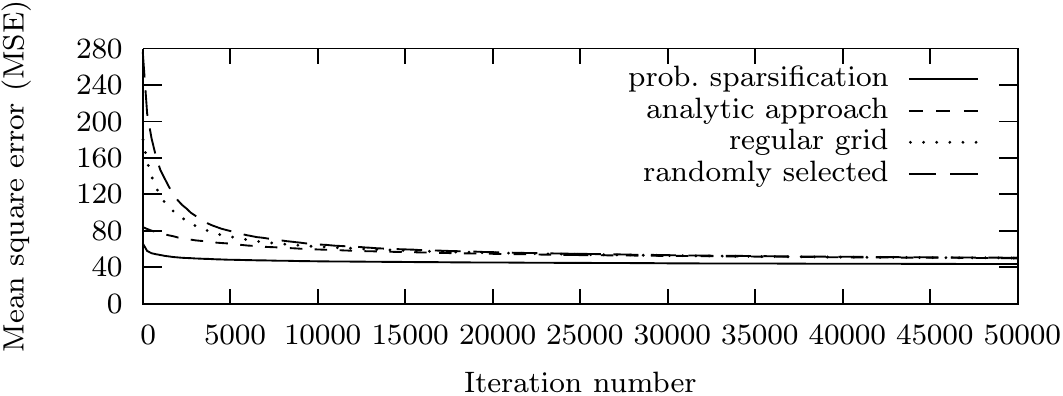}
  \caption{Convergence behaviour for the first 50{,}000 iterations, when the
    nonlocal pixel exchange is applied to different masks (cf.
    Figure~\ref{fig:results}(a,d,g)) of the trui test image with 
    optimal parameter $m$ (cf. Table \ref{tab:mntest}).}
  \label{fig:nlpe-test-masks}
\end{figure}
 
% ---------------------------------------------------------------------------- 

\subsection{Optimising Tonal Data}
\label{sec:optim-tonal-data}

So far, all our 2D optimisation approaches only propose a solution for 
the spatial optimisation problem. However, the results in 
Section~\ref{sec:optim-knots-appr} show that a tonal optimisation,
i.e.\ an optimisation of the grey values, can be very worthwhile. From a data
compression point of view, it is important to notice that changing the grey
values at the chosen data points does not increase the amount of data that 
needs to be stored. The quality improvements, on the other hand, can be 
remarkable. In this section, we present an approach that allows us to 
determine the optimal grey values for any given mask.

In order to find the optimal grey values $\bm{g}$ for a fixed mask $\bm{c}$, 
we consider the following minimisation approach:
\begin{equation}
  \label{eq:gvo-ansatz}
  \argmin_{\bm{g}}\{|\bm{f} - \inp{\bm{g}}|^2\}\;,
\end{equation}
where $|\cdot|$ is the standard Euclidean norm, $\bm{f}$ denotes the original
image, and $\inp{\bm{g}}$ the reconstruction from
\eqref{eq:inpainting_function}. Due to the linearity of $\bm{r}$ with 
respect to $\bm{g}$, this is a linear least squares problem. In our next
steps, we analyse its well-posedness properties and propose an efficient 
numerical algorithm. 

%---------------------------------------------------------------------------- 

\subsubsection{Existence and Uniqueness Results}
\label{sec:exist-uniq-results}

Let $\bm{e_i}$ denote the $i$-th canonical basis vector of 
$\mathbb{R}^{|J|}$, i.e.\ $e_{i,j}=1$ if $i=j$, and $0$ otherwise. Then 
we call $\inp{\bm{e_i}}$ the \emph{inpainting echo} of pixel $i$. 
Since $\bm{r}$ is linear in $\bm{g}$ we can express the reconstruction 
$\bm{u}$ as a superposition of its inpainting echoes: 
\begin{equation}
  \bm{u} \;=\; \inp{\bm g} 
         \;=\; \bm{r} \Big(\bm{c}, \sum_{i\in J} g_i \bm{e_i}\Big) 
         \;=\; \sum_{i\in J} g_i \, \inp{\bm{e_i}}.
\end{equation}
Since $\inp{\bm{e_i}}$ is $\bm 0$ for $c_i = 0$ (i.e. 
for $i \in J\setminus K$), we can simplify this to a summation over $K$:
\begin{equation}
  \inp{\bm g} = \sum_{i\in K} \, g_i \inp{\bm{e_i}}.
\end{equation}
For our minimisation problem \eqref{eq:gvo-ansatz}, this means that 
the coefficients $g_i$ can be chosen arbitrarily if $i \in J\setminus K$. 
For simplicity, we fix them at 0.
The remaining $g_i$ with $i \in K$ can be obtained by considering the least
squares problem
\begin{equation}
  \label{eq:normal_equations}
  \argmin_{\bm{g}_K} \left\{|\bm{B}\bm{g}_K - \bm f|^2\right\}
\end{equation}
where $\bm{g}_K = (g_i)_{i\in K}$ is a vector of size $|K|$, and $\bm{B}$ is a
$|J|\times|K|$ matrix that contains the vectors 
$\{\inp{\bm{e_i}} \:|\; i\in K\}$ as columns. 
The associated normal equations are given by
\begin{equation}
  \label{eq:normal-eqs}
  \bm{B}^\top \bm{B}\, \bm{g}_K = \bm{B}^\top \bm{f}.
\end{equation}
By construction, the $|K| \times |K|$ matrix $\bm{B}^\top \bm{B}$ is
positive semidefinite. It is, however, not obvious that the eigenvalue
$0$ cannot appear. The theorem below excludes such a singular situation.

\begin{theorem}
Let $K$ be nonempty. Then the matrix $\bm{B}^\top \bm{B}$ is invertible,
and thus the linear system \eqref{eq:normal-eqs} has a unique solution. 
\end{theorem}

\begin{proof}
Mainberger {\em et al.}~\cite{MBWF11} have proven that for a nonempty set $K$, 
the $|J| \times |J|$ matrix $\bm{M} = \bm{C} - (\bm{I}-\bm{C}) \bm{A}$ 
is invertible, i.e. $\bm{M}^{-1}$ exists. Moreover, for $i\in K$ we have
\begin{equation}
  \inp{\bm{e_i}} = \bm{M}^{-1} \bm{C}\bm{e_i}
  \stackrel{i\in K}{=} \bm{M}^{-1}\bm{e_i}.
\end{equation}
This shows that $\inp{\bm{e_i}}$ is the $i$-th column of $\bm{M}^{-1}$.
Since $\bm{M}$ is invertible, also $\bm{M}^{-1}$ is regular. Thus, all 
column vectors of $\bm{M}^{-1}$ have to be linearly independent. In
particular, this implies that also all columns 
$\{\inp{\bm{e_i}} \:|\; i\in K\}$ of the $|J| \times |K|$ matrix $\bm{B}$ 
are linearly independent. Therefore, $\bm{B}^\top \bm{B}$ 
is invertible, and the linear system \eqref{eq:normal-eqs} has a unique 
solution.
\qed{}
\end{proof}

%............................................................................

\subsubsection{Numerical Algorithm}
\label{sec:numerical-algorithm}

To find the solution of the minimisation problem~\eqref{eq:gvo-ansatz}, 
Mainberger \emph{et al.}~\cite{MHWT12} have solved the associated normal 
equations~\eqref{eq:normal-eqs} iteratively. In particular, they have chosen 
a randomised Gau{\ss}--Seidel scheme that updates the grey values at the 
individual mask 
points one after another. Alternatively, one could also employ different 
iterative solvers or solve the system of equations with direct methods 
such as LU- or QR-decompositions~\cite{Hi02}. 
In general, however, one has to state that typical methods which rely on 
inpainting echoes suffer from a relatively high computational cost to 
obtain the individual echoes. Although one can precompute and reuse them 
for the subsequent iteration steps, they still need to be computed at 
least once. This is particularly inefficient when a large amount of 
mask points is present. Moreover, precomputing inpainting echoes leads
to higher memory requirements. 
The more recent strategies in~\cite{CRP14,HW15} avoid the direct 
computation of the inpainting echoes. Instead, the original minimisation 
problem \eqref{eq:gvo-ansatz} is solved directly with the help of 
primal-dual methods or related sophisticated optimisation strategies. 
This gives more efficient algorithms for tonal optimisation.

In the following, we propose an alternative approach to find optimal 
tonal data. It is based on an accerated gradient descent strategy
that allows an efficient grey value optimisation. Similar to 
\cite{CRP14,HW15}, we consider the original minimisation problem 
\eqref{eq:gvo-ansatz} directly. Thus, we also avoid computing 
inpainting echoes. This leads to a reduced runtime as well 
as to low memory requirements. Below we first explain the classical
gradient method with exact line search. Afterwards we present a
novel variant that benefits from an accelation with a fast
explicit diffusion scheme in the sense of Grewenig {\em et al.}
\cite{GWB10}.

Our goal is to minimise the objective function
\begin{equation}
 E(\bm{g}) = \frac{1}{2}\, |\inp{\bm{g}} - \bm{f}|^2.
 \label{eq:error}
\end{equation}
Starting with an initialisation $\bm{g}^0=\bm{C}\bm{f}$, a gradient descent
scheme minimises this energy $E$ by iteratively updating the
current grey values for $k>0$ as
\begin{equation}
 \label{eq:gd}
\bm{g}^{k+1} = \bm{g}^k - \alpha\, \bm{\nabla} E(\bm{g}^k)
\end{equation}
with a step size $\alpha$ and the gradient $\bm{\nabla} E(\bm{g}^k)$ 
depending on the iterates $\bm{g}^k$. Denoting the inpainting solution 
at iteration step $k$ by $\bm{u}^k$, i.e. $\bm{u}^k := \inp{\bm{g}^k}$, 
the gradient can be written as
\begin{equation}
  \label{eq:2}
\bm{\nabla} E(\bm{g}^k) = \bm{J}^\top \left(\bm{u}^k - \bm{f}\right)
\end{equation}
where $\bm{J}$ is the Jacobian of $\inp{\bm{g}^k}$ with respect to the second
component. With the definition of $\inp{\bm{g}^k}$ from
\eqref{eq:inpainting_function}, we obtain
\begin{align}
  \label{eq:3}
  \bm{J}^\top &= \left(\left(\bm{C}-(\bm{I}-\bm{C})\bm{A}\right)^{-1}
           \bm{C}\right)^\top \nonumber\\
         &= \bm{C}\left(\bm{C}-\bm{A}(\bm{I}-\bm{C})\right)^{-1}
\end{align}
where we have exploited the symmetries of the matrices $\bm{C}$ and $\bm{A}$.
Computing the iterates $\bm{u}^k$ and the gradient $\bm{\nabla} E(\bm{g}^k)$
means to solve a linear system of equations for each of them. This can be 
done efficiently with a bidirectional multigrid solver as is suggested in 
\cite{MBWF11}.

The main parameter that we have to specify is the step size $\alpha$. As one 
possibility, it can be optimised to yield the largest possible decay
of the energy in each step. This comes down to the least squares problem 
\begin{equation}
\argmin_{\alpha>0} \left\{|\bm{f} -
          \inp{\bm{g}^k\!-\!\alpha\bm{\nabla} E(\bm{g}^k)}\!|^2\right\}.
\end{equation}
Exploiting the linearity of $\bm{r}$ in the second component allows 
to obtain the minimiser in closed form:
\begin{equation}
  \label{eq:5}
\alpha = \frac{(\bm{f}-\bm{u}^k)^\top \inp{\bm{\nabla} E(\bm{g}^k)}}
              {|\inp{\bm{\nabla} E(\bm{g}^k)}|^2}\,.
\end{equation}
This strategy is also known as exact line search. It guarantees that 
the sequence $(\bm{g}^k)_k$ converges to the minimum of the 
energy~\cite{BV04}. As stopping criterion, we consider the relative 
norm of the gradient, which should approach zero at the optimum. 
In other words, we stop as soon as
\begin{equation}
 \label{eq:6}
% \frac{|\bm{\nabla} E(\bm{g}^k)|^2}{|\bm{\nabla} E(\bm{g}^0)|^2} 
% < \varepsilon
  |\bm{\nabla} E(\bm{g}^k)|^2 \le \varepsilon \, |\bm{\nabla} E(\bm{g}^0)|^2 
\end{equation}
with some small number $\varepsilon>0$. An algorithmic overview of the
classical gradient descent method with exact line search is shown 
below. It serves as our baseline method.\par
 
%---------------------------------------------------------------------------
\pagebreak
\begin{samepage}
  \begin{description}
    \item[\textbf{Algorithm 6:}] \textbf{Grey Value Optimisation with Exact Line Search}\\[-3mm]
          \hrule\vspace{2mm}\nopagebreak
    \item[\emph{Input:}]\hfill\\
          Original image $\bm{f}$, pixel mask $\bm{c}$.
    \item[\emph{Initialisation:}]\hfill\\
          $\bm{g}^0\coloneqq\bm{C}\bm{f}$, \; $k\coloneqq 0$.
    \item[\emph{Compute:}]\hfill\\
          Repeat for $k \ge 0$
          \begin{enumerate}
            \item Compute the gradient
                  \begin{equation}
                    \bm{\nabla} E(\bm{g}^k) \coloneqq \bm{J}^\top 
                    \left(\inp{\bm{g}^{k}} - \bm{f}\right).
                  \end{equation}
            \item Determine the step size $\alpha$ with Equation~\eqref{eq:5}.
            \item Update the tonal data:
                  \begin{equation}
                    \bm{g}^{k+1} \coloneqq \bm{g}^{k} - \alpha\, \bm{\nabla} E(\bm{g}^k).
                  \end{equation}
          \end{enumerate}
          until the stopping criterion \eqref{eq:6} is fulfilled.
    \item[\emph{Output:}]\hfill\\
          Optimised grey values $\bm{g}$.\\
          \vspace*{-2mm}\hrule
  \end{description}
\end{samepage}
%---------------------------------------------------------------------------

In order to speed up the gradient descent approach, we propose an 
accelerated algorithm based on a so-called \emph{fast explicit diffusion 
(FED)} scheme. First applications of FED to image processing problems go 
back to Grewenig {\em et al.}~\cite{GWB10}. FED can be used to speed up 
any explicit diffusion-like algorithm that involves a symmetric matrix. 
While classical explicit schemes employ a constant time step size
that has to satisfy a restrictive stability limit, FED schemes involve
cycles of time step sizes where up to 50 \% of them can violate this
stability limit. Nevertheless, at the end of each cycle, stability in
the Euclidean norm is achieved. In contrast to classical explicit schemes 
that reach a stopping time of order $O(M)$ in $M$ steps, FED schemes with 
cycle length $M$ progress to $O(M^2)$. This allows a very substantial
acceleration.

Since the gradient descent scheme can be seen as an explicit scheme with
a symmetric matrix, FED is applicable: If $\alpha^*$ denotes a fixed step 
size for which \eqref{eq:gd} is stable in the Euclidean norm, one replaces 
it by the cyclically varying step sizes
\begin{equation}
 \label{eq:ai}
 \alpha_i \,=\, \alpha^* \cdot
              \dfrac{1}{2\cos^2\left(\pi\cdot\frac{2i+1}{4M+2}\right)}
 \quad (i=0,...,M\!-\!1).
\end{equation}
For large cycle lengths $M$, one should permute the order of the step sizes 
to tame rounding errors; see \cite{WGSB15} for more details on this and an 
exhaustive explanation of the FED framework in general.
The FED cycles should be iterated until the stopping criterion \eqref{eq:6} 
is fulfilled. A related cyclic optimisation strategy has also been 
investigated in~\cite{SSM13}.

In order to determine the individual step sizes within each cycle, we 
have to find a step size $\alpha^*$ that is within the stability limit. 
The following result is well-known in optimisation theory \cite{NN04}: 
If $E(\bm{g})$ is continuous and its gradient is Lipschitz continuous, 
i.e. there is a constant $L$ sucht that
\begin{equation}
\left| \bm{\nabla}E(\bm{g}_1) - \bm{\nabla}E(\bm{g}_2) \right|
   \leq L \cdot \left| \bm{g}_1 - \bm{g}_2 \right|
\end{equation}
for all $\bm{g}_1$ and $\bm{g}_2$, then the gradient descent scheme is 
stable for all step sizes $\alpha^*$ fulfilling
\begin{equation}
0 < \alpha^* < \frac{2}{L}.
\end{equation}
It is straightforward to verify that in our case $L$ can be chosen as 
the squared spectral norm of the inpainting matrix 
$\bm{D}:=\bm{M}^{-1} \bm{C}$, i.e.
\begin{equation}
 \label{eq:spectral_norm}
 L = |\bm{D}|^2 := \rho(\bm{D}^\top \bm{D}),
\end{equation}
where $\rho(\bm{D}^\top \bm{D})$ denotes the spectral radius of the symmetric 
matrix $\bm{D}^\top \bm{D}$. One possibility to estimate the spectral radius 
is to use Gershgorin's circle theorem. However, this may give a too 
pessimistic estimate. Instead, we propose to use the power method to determine 
the maximum eigenvalue of $\bm{D}^\top \bm{D}$; see e.g.~\cite{ATK08}. The 
convergence of this method turns out to be relatively fast, such that one 
obtains already a reasonable estimate of the spectral radius after $5$ 
iterations.

Although it may appear tempting to choose $\alpha^*$ close to the 
stability limit $\frac{2}{L}$, this can result in a suboptimal convergence
speed, since high frequent error components are damped too slowly. 
Our experiments suggest that a good choice for $\alpha^*$ is two third of 
the stability limit: 
\begin{equation}
 \alpha^* = \frac{4}{3L}.
\end{equation}
Similar strategies are also common e.g.~in multigrid approaches that
use a damped Jacobi method with damping factor $\frac{2}{3}$ as a
baseline solver \cite{Bri87}.

Below we give an algorithmic overview over all steps to perform tonal 
optimisation with FED-accelerated gradient descent:

%---------------------------------------------------------------------------
  \begin{description}
    \item[\textbf{Algorithm 7:}] \textbf{Grey Value Optimisation with FED}\hfill
          \hrule\vspace{2mm}
    \item[\emph{Input:}]\hfill\\
          Original image $\bm{f}$, pixel mask $\bm{c}$, FED cycle length $M$. 
    \item[\emph{Initialisation:}]\hfill\\
          $\bm{g}^0\coloneqq\bm{C}\bm{f}$, \; $k\coloneqq 0$.
    \item[\emph{Compute:}]\hfill
          \begin{enumerate}
            \item Estimate $L$ in~\eqref{eq:spectral_norm} with the power method.
            \item Determine the FED time steps $\alpha_0,\ldots,\alpha_{M-1}$ 
                  according to \eqref{eq:ai} with $\alpha^*=\frac{4}{3L}$. 
                  If necessary, permute them.
            \item Repeat for $k \ge 0$
                  \begin{enumerate}
                    \item $\bm{g}^{k,0}:=\bm{g}^{k}$ 
                    \item For $i=0,\ldots,M-1$ do
                          \begin{enumerate}
                            \item Compute the gradient
                                  \begin{equation}
                                    \bm{\nabla} E(\bm{g}^{k,i}) \coloneqq \bm{J}^\top 
                                    \left(\inp{\bm{g}^{k,i}} - \bm{f}\right).
                                  \end{equation}
                            \item Update the tonal data:        
                                  \begin{equation}
                                    \bm{g}^{k,i+1} \coloneqq \bm{g}^{k,i} - \alpha_i\, 
                                    \bm{\nabla} E(\bm{g}^{k,i}).
                                  \end{equation}
                          \end{enumerate}
                    \item $\bm{g}^{k+1}:=\bm{g}^{k,M}$ 
                  \end{enumerate}
          \end{enumerate}
          until the stopping criterion \eqref{eq:6} is fulfilled.
    \item[\emph{Output:}]\hfill\\
          Optimised grey values $\bm{g}$.\\
          \vspace*{-2mm}\hrule
  \end{description}
% ---------------------------------------------------------------------------- %

Since the grey value optimisation problem is strictly convex, all 
convergent algorithms yield the same minimiser and are therefore 
qualitatively equivalent. They only differ by their run times.
The following experiment gives an impression of realistic run times 
of both gradient descent algorithms for greyvalue optimisation. As
before, we consider the inpainting problem with the test image {\em trui} 
($256 \times 256$ pixels) and 4 \% mask density. The stopping parameter 
for our iterations was set to $\varepsilon :=0.001$. With a C 
implementation on a desktop PC with Intel Xeon processor (3.2GHz), the 
exact line search algorithm requires $458$ seconds to perform $262$ 
iterations. A corresponding FED algorithm with $\alpha^*=0.01$ needs 
only $77$ seconds to compute $4$ cycles of length $M=15$.
In comparison, a CPU implementation of the primal--dual approach of 
Hoeltgen and Weickert \cite{HW15} requires for the same problem a 
run time of $346$ seconds. This illustrates the favourable performance 
of the FED-accelerated gradient descent method.\par{}
As the FED algorithm is based on an explicit scheme, it is also 
well-suited for implementations on parallel hardware such as GPUs. 
This can lead to very substantial additional accelerations.

%.............................................................................

\subsubsection{Qualitative Evaluation}
\label{sec:evaluation}

In order to evaluate the capabilities of the grey value optimisation, 
we apply it to the masks obtained so far for the test image \emph{trui}. 
In all cases this results in a clear improvement of the reconstruction
quality, both visually and in terms of their MSE; see Fig.~\ref{fig:results} 
as well as Tab.~\ref{tab:results}. This confirms also our impressions from 
the 1D scenario in Section~\ref{sec:optim-knots-appr}. Especially suboptimal 
masks benefit a lot from grey value optimisation: It can compensate for
many deficiencies that are caused by an inferior spatial selection strategy. 
This becomes particularly apparent when considering the result for the random 
mask or the mask on the regular grid. It is remarkable how much gain in 
quality is possible by simply choosing different grey values. 
However, also the reconstruction we obtain with the best mask using 
probabilistic sparsification and nonlocal pixel exchange can be improved
substantially: The MSE is reduced from $41.92$ to $27.24$.
Similar improvements can also be observed for the images \emph{walter} and 
\emph{peppers}; see Tab.~\ref{tab:results}. The best reconstructions are 
depicted in Fig.~\ref{fig:resultswalterpeppers}.

%...........................................................................

\begin{figure*}
  \centering{}
  \resizebox{\textwidth}{!}{
 \begin{tabular}{C@{\hspace{2mm}}EEE} 
 \toprule & mask & reconstruction & reconstruction with optimal tonal data\\[-1mm]
 \midrule \\[-4mm]
 random mask
   & \includegraphics[width=0.17\textwidth]{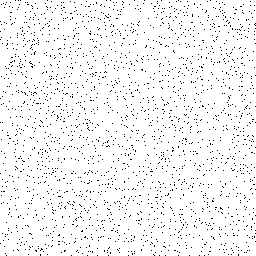} 
   & \includegraphics[width=0.17\textwidth]{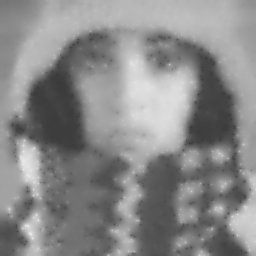} 
   & \includegraphics[width=0.17\textwidth]{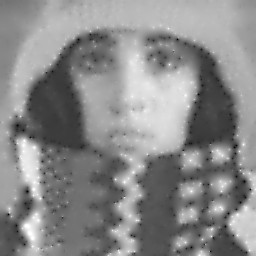} \\
   & & MSE: 273.10 & MSE: 151.25\\[-1mm]
 \midrule \\[-4mm]
 regular grid
   & \includegraphics[width=0.17\textwidth]{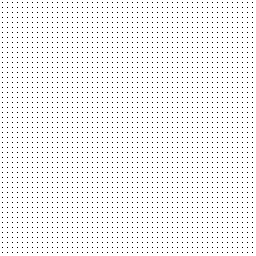} 
   & \includegraphics[width=0.17\textwidth]{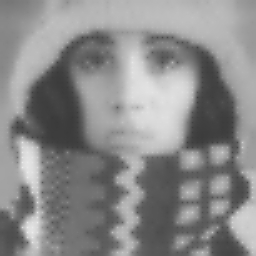} 
   & \includegraphics[width=0.17\textwidth]{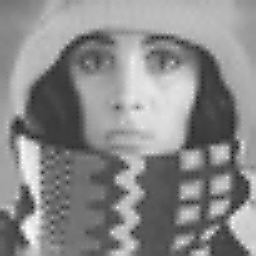} \\
   & & MSE: 181.72 & MSE: 101.62\\[-1mm]
 \midrule \\[-4mm]
 anal. approach
   & \includegraphics[width=0.17\textwidth]{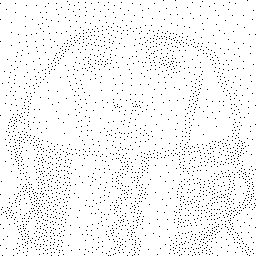} 
   & \includegraphics[width=0.17\textwidth]{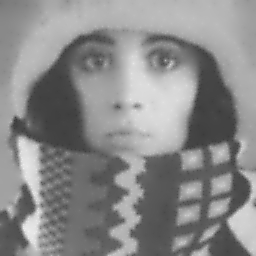} 
   & \includegraphics[width=0.17\textwidth]{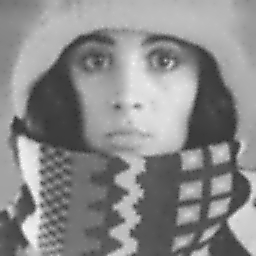} \\
   & & MSE: 84.04 & MSE: 42.00 \\[-1mm] 
 \midrule \\[-4mm]
 probab.~sparsif.
   & \includegraphics[width=0.17\textwidth]{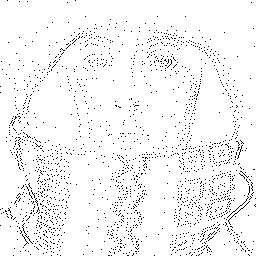} 
   & \includegraphics[width=0.17\textwidth]{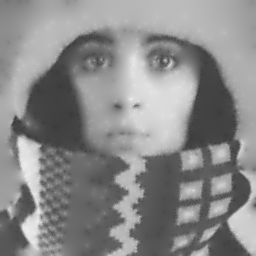} 
   & \includegraphics[width=0.17\textwidth]{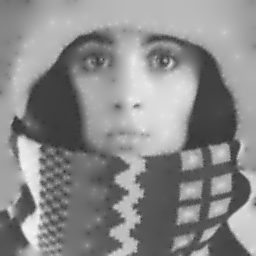} \\
   & & MSE: 66.11 & MSE: 36.04 \\[-1mm]
 \midrule \\[-4mm]
 nonl. pixel exch.
   & \includegraphics[width=0.17\textwidth]{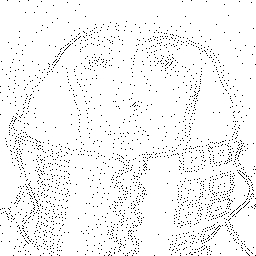} 
   & \includegraphics[width=0.17\textwidth]{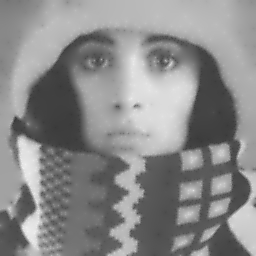} 
   & \includegraphics[width=0.17\textwidth]{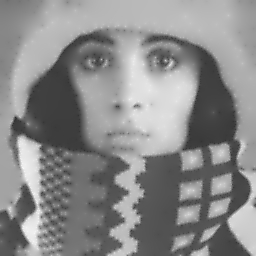} \\
  & & MSE: 41.92  & MSE: 27.24 \\[-1mm]
 \bottomrule 
 \end{tabular}
 }
\caption{Evaluation of different inpainting data using 4\% of all pixels.
 \textbf{Left column:} Different masks obtained by using a regular grid, 
   the analytic approach ($s=0.80$, $\sigma = 1.6$), 
   a probabilistic sparsification ($p=0.3$, $q=10^{-6}$),
   and with an additional nonlocal pixel exchange ($m=20$, $5\cdot{}10^{5}$ 
   iterations) after the previous probabilistic sparsification.
 \textbf{Middle column:} Reconstructions with homogeneous diffusion inpainting
   with the masks from the first column. 
 \textbf{Right column:} Same as in the middle column, but with optimal tonal 
   data.}
 \label{fig:results}
\end{figure*}
 
%...........................................................................

\begin{table*}
 \setlength{\tabcolsep}{1em}
 \centering{}
 \caption{Comparison of the reconstruction error (MSE) with 4\% of all pixels
   for different test images and different inpainting data. 
   The parameters were applied to both approaches, with and without grey value 
   optimisation (GVO). The nonlocal pixel exchange was performed with 
   $5\cdot{}10^{5}$ iterations for each experiment. 
%   Our combination of 
%   probabilistic sparsification with nonlocal pixel exchange can outperform 
%   the regular selection of data points as well as the analytic approach in 
%   each case.
   }\label{tab:results}
\resizebox{\textwidth}{!}{
 \begin{tabular}{cccccccccc}\toprule
 \multirow{2}{*}{image}
   & \multirow{2}{*}{GVO} 
   & randomly 
   & regular 
   & \multicolumn{2}{c}{analytic} 
   & \multicolumn{2}{c}{probabilistic}         
   & \multicolumn{2}{c}{nonlocal}\\ 
   &                      
   & selected 
   & grid    
   & \multicolumn{2}{c}{approach} 
   & \multicolumn{2}{c}{sparsification}        
   & \multicolumn{2}{c}{pixel exchange}\\ 
   \cmidrule(r){1-1} 
   \cmidrule(lr){2-2} 
   \cmidrule(lr){3-3} 
   \cmidrule(lr){4-4} 
   \cmidrule(lr){5-6} 
   \cmidrule(lr){7-8} 
   \cmidrule(lr){9-10} 
 \multirow{2}{*}{\emph{trui}}  
   & no  
   & 273.10  
   & 181.72 
   & 84.04 
   & \multirow{2}{1.5cm}{($\sigma=1.6$, \hspace*{1ex}$s=0.80$)} 
   & 66.11 
   & \multirow{2}{1.7cm}{($p=0.3$, \hspace*{1ex}$q=10^{-6}$)} 
   & 41.92          
   & \multirow{2}{*}{($m=20$)} \\ 
   & yes                  
   & 151.25   
   & 101.62  
   & 42.00                        
   & 
   & 36.04 
   & 
   & \textbf{27.24} 
   &  \\ \addlinespace{}
 \multirow{2}{*}{\emph{walter}}  
   & no                   
   & 297.07   
   & 184.00  
   & 39.85                        
   & \multirow{2}{1.5cm}{($\sigma=1.5$, \hspace*{1ex}$s=1.00$)} 
   & 32.96 & \multirow{2}{1.7cm}{($p=0.2$, \hspace*{1ex}$q=10^{-6}$)} 
   & 18.37          
   & \multirow{2}{*}{($m=30$)} \\
   & yes                  
   & 155.50   
   & 91.97   
   & 20.19                        
   & 
   & 19.24 
   &  
   & \textbf{12.45} 
   & \\ \addlinespace{}
 \multirow{2}{*}{\emph{peppers}} 
   & no                   
   & 278.61   
   & 185.04  
   & 70.05                        
   & \multirow{2}{1.5cm}{($\sigma=1.5$, \hspace*{1ex}$s=0.95$)} 
   & 44.85 & \multirow{2}{1.7cm}{($p=0.1$, \hspace*{1ex}$q=10^{-6}$)} 
   & 29.63          
   & \multirow{2}{*}{($m=30$)} \\
   & yes  
   & 156.15   
   & 104.41  
   & 43.77   
   & 
   & 28.58 
   &  
   & \textbf{25.10} 
   & \\ \bottomrule
\end{tabular}}
\end{table*}

%%%%%%%%%%%%%%%%%%%%%%%%%%%%%%%%%%%%%%%%%%%%%%%%%%%%%%%%%%%%%%%%%%%%%%%%%%%%%

\section{Extensions to Other Inpainting Operators}
\label{sec:extensions}
 
We have seen that optimising the interpolation data allows to obtain high 
quality reconstructions with only $4\%$ of all pixels. These results are 
also remarkable in view of the fact that so far we have used a very simple 
interpolation operator: the Laplacian. It is unlikely that it offers the
best performance. In this section we investigate the question if the 
algorithms of Section~\ref{sec:2-d-optimisation} can be used with, or
extended to more advanced inpainting operators. 
We focus on two representative operators that have proven their good 
performance in the context of PDE-based image compression with sparse data 
\cite{GWWB05,GWWB08,CRP14,SPME14}: the biharmonic operator and the EED 
operator.

A straightforward extension of the Laplacian is the biharmonic operator,
i.e.\ we replace $\Delta u$ in \eqref{eq:ci1} by $-\Delta^2 u$. Using it for
interpolation comes down to thin plate spline interpolation \cite{Du76}, a
rotationally invariant multidimensional generalisation of cubic spline
interpolation. Compared to the Laplace operator, it yields a smoother solution
$u$ around the interpolation data, since its Green's function is twice
differentiable. This avoids the typical singularities that distort the 
visual quality with homogeneous diffusion inpainting. These artifacts 
are caused by the logarithmic singularities of the Green's function of
the two-dimensional Laplacian. On the other hand, biharmonic inpainting 
is prone to over- and undershoots, i.e. the values of $u$ may leave the range 
of the inpainting data in $K$. This cannot happen for homogeneous diffusion 
inpainting which fulfils a maximum--minimum principle. Nevertheless, a
number of evaluations show that biharmonic inpainting can offer a better 
reconstruction quality than homogeneous diffusion inpainting 
\cite{GWWB05,GWWB08,CRP14,SPME14}.\par{}
Secondly we consider an anisotropic nonlinear diffusion operator, namely 
{\em edge enhancing-diffusion (EED)}. Originally it has been introduced for 
image denoising \cite{We94e}, and its application in image compression goes 
back to Gali\'c {\em et al.}~\cite{GWWB05}. Using EED means that we replace 
$\Delta u$ in 
\eqref{eq:ci1} by $\operatorname{div}(\bm{D}(\bm{\nabla} u_\sigma) 
\bm{\nabla} u)$. The positive definite matrix $\bm{D}(\bm{\nabla} u_\sigma)$ 
is the so-called diffusion tensor. It steers the diffusion process by its 
eigenvectors and eigenvalues. They depend on the gradient of a 
Gaussian-smoothed version $u_\sigma$ of the image $u$, where $\sigma$ 
denotes the standard deviation of the Gaussian. The first 
eigenvector of $\bm{D}$ is chosen to be orthogonal to $\bm{\nabla} u_\sigma$, 
and the corresponding eigenvalue is fixed at $1$. This gives full diffusion /
inpainting along image edges. In contrast, the second eigenvector
is chosen to be parallel to $\bm{\nabla} u_\sigma$, and its eigenvalue is
a decreasing function of the local image contrast $|\bm{\nabla} u_\sigma|$.
Thus, one reduces diffusion / inpainting across high contrast edges. For 
image compression, one usually chooses the diffusivity of 
Charbonnier {\em et al.} \cite{CBAB97}, which is given by 
$\left(1 + |\bm{\nabla} u_\sigma|^2/\lambda^2\right)^{-1/2}$. 
The parameter $\lambda > 0$ allows to steer the contrast dependence.\par{}
Image inpainting with EED can reconstruct edges in high quality, even 
when the specified data is sparse \cite{SPME14}. This explains why EED 
has become one of the best performing operators for PDE-based image 
compression \cite{GWWB05,GWWB08,SPME14}. 
Moreover, for second order elliptic differential operators such as EED, the
continuous theory guarantees a maximum--minimum principle \cite{We97}.
Experiments show that in contrast to homogeneoues diffusion inpainting, 
EED does not suffer from singularities \cite{SPME14}.\par{}
What are the changes in our framework, when we replace the Laplacian 
by the biharmonic or the EED operator?
If we discretise the biharmonic operator with central finite 
differences, we have to exchange the matrix $\bm A$ in \eqref{eq:discinp}
by another constant matrix, but the structure of this equation remains the 
same as for homogeneous diffusion inpainting: It is a linear system of
equations in the unknown vector $\bm{u}$. For EED, however, the discrete
differential operator $\bm{A}$ depends nonlinearly on the reconstruction 
$\bm{u}$. Thus, the resulting system of equations becomes nonlinear:
\begin{equation}
 \label{eq:discinp2a}
 \bm{C}\left(\bm{u}-\bm{f}\right) -
 \left(\bm{I} - \bm{C}\right) \bm{A}(\bm{u})\,\bm{u} = \bm{0}.
\end{equation}
While this nonlinearity does not affect our spatial optimisation, we will 
see that it makes the tonal optimisation more difficult.
 
% ---------------------------------------------------------------------------- 

\subsection{Optimising Spatial Data}
\label{sec:ext_spatial}
 
Let us consider the spatial optimisation first. The theory behind the 
analytic approach is strongly based on homogeneous diffusion, and no
extensions to more sophisticated inpainting operators have been 
discussed in \cite{BBBW08}. Therefore, we do not consider this 
approach in this section.

In contrast, the probabilistic approach can be used without any restrictions 
and changes for both the biharmonic operator as well as for EED. As before, 
we search for the best parameters of the spatial optimisation methods. 
In all our experiments, we fix the EED parameters to $\lambda=0.8$ and 
$\sigma=0.7$, since this gives reconstructions of high quality.\par{}

Results for both operators and methods in comparison to homogeneous diffusion
inpainting are shown in Tab.~\ref{tab:ext_results}. Interestingly, the 
biharmonic operator does not achieve better results than the homogeneous 
one in the case of pure probabilistic sparsification without nonlocal 
pixel exchange. Our explanation for this observation is as follows: 
Biharmonic inpainting has a positive effect of increased smoothness 
around data points, and a negative effect by over- and undershoots in 
the inpainting domain. The latter one is ignored in the 
sparsification decision, since the point selection is based on a purely 
local error measurement. 
Thus, for very sparse data, over- and 
undershoots can become particularly detrimental on global error measures
such as the MSE. 
In our experiment, the minimum and maximum values of 
the reconstruction lie around $-59$ and $346$, respectively. This is far 
beyond the range of the original image which is given by $[56, 214]$. 
% However, experiments indicate that for a larger number of mask pixels, 
% the biharmonic operator can outperform the homogeneous one when
% using probabilistic sparsification. 
With an additional nonlocal pixel exchange, we can overcome this problem, 
since it uses the MSE as criterion to redistribute the mask pixels to 
more favourable locations. 
As a consequence, the grey value range of the reconstruction shrinks to 
the interval $[47,248]$, and the MSE falls from $79.26$ to $20.89$.
This is far better than the MSE of $41.92$ that we achieve for homogeneous 
diffusion inpainting.

However, EED with probabilistic sparsification gives far superior results. 
Already without nonlocal pixel exchange, we obtain an MSE of $24.20$. 
Postprocessing with nonlocal pixel exchange reduces the error to only 
$12.62$.

Figures~\ref{fig:ext_results}(a) and (c) depict the masks obtained with
probabilistic sparsification combined with the nonlocal pixel exchange for both
operators. Comparing them with the one from homogeneous diffusion inpainting
(see Fig.~\ref{fig:results}), we observe interesting differences: Homogeneous 
diffusion inpainting requires more pixels near edges to represent these 
discontinuities well. Thus, for a specified pixel number, it has less 
pixels to approximate the interior of the regions. This explains why 
biharmonic or EED inpainting can achieve lower errors.
% 
% Finally, Fig.~\ref{fig:ext_results} depicts the best masks and reconstructions
% for both operators, i.e., after applying probabilistic sparsification combined
% with a nonlocal pixel exchange. 

%............................................................................
 
\begin{table*}
  \setlength{\tabcolsep}{1em}
\caption{
  Comparison of the reconstruction error (MSE) with 4\% of all pixels for 
  different inpainting operators and different data optimisation algorithms. 
  The parameters were applied to both approaches, with and without grey value 
  optimisation (GVO). The nonlocal pixel exchange was performed with 
  $5 \cdot 10^5$ iterations for each experiment.}
  \label{tab:ext_results}
 \centering
\resizebox{\textwidth}{!}{
\begin{tabular}{@{}lllll}\toprule
 operator & 
   reg. grid & 
   probab. sparsif. & 
   nonl. pixel exch. & 
   with GVO \\ 
 \cmidrule(r){1-1} 
   \cmidrule(lr){2-2} 
   \cmidrule(lr){3-3} 
   \cmidrule(lr){4-4} 
   \cmidrule(l){5-5}
 homogeneous & 
   181.72 & 
   66.11\ ($p=0.3$, $q=10^{-6}$) & 
   41.92\ ($m=20$)   & 
   \textbf{27.24}\\
 biharmonic & 
   107.96 & 
   79.26\ ($p=0.005$, $q=10^{-6}$) & 
   20.89\ ($m=10$) & 
   \textbf{16.73}\\
 EED ($\lambda=0.8$, $\sigma=0.7$) & 
   102.85 & 
   24.20\ ($p=0.05$, $q=10^{-6}$) & 
   12.62\ ($m=30$) & 
   \textbf{10.79}\\ 
 \bottomrule
\end{tabular}}
\end{table*}

%............................................................................
 
\begin{figure*}
\centering
\begin{minipage}{0.235\textwidth}
 \centering
 \includegraphics[width=\textwidth]{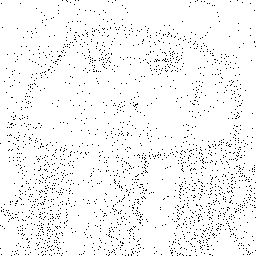}\\
 \textbf{(a)}
\end{minipage}\hfill
\begin{minipage}{0.235\textwidth}
 \centering
 \includegraphics[width=\textwidth]{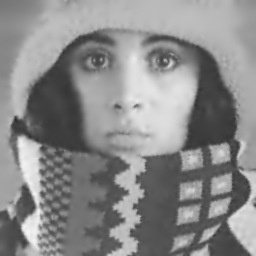}\\
 \textbf{(b)}
\end{minipage}\hfill
\begin{minipage}{0.235\textwidth}
 \centering
 \includegraphics[width=\textwidth]{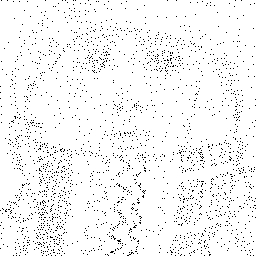}\\
 \textbf{(c)}
\end{minipage}\hfill
\begin{minipage}{0.235\textwidth}
 \centering
 \includegraphics[width=\textwidth]{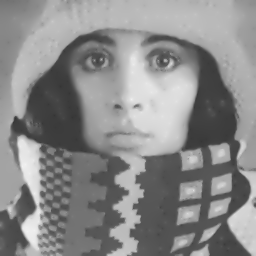}\\
 \textbf{(d)}
\end{minipage}
\caption{
  Best mask and reconstruction for the biharmonic operator \textbf{(a, b)} 
  and EED \textbf{(c, d)} for the test image \emph{trui} with $4$ \% of all
  pixels, using probabilistic sparsification, nonlocal pixel exchange, 
  and tonal optimisation. The image (b) has an MSE of $16.73$, while image 
  (d) has an MSE of $10.79$. Parameters have been chosen as in Table 
  \ref{tab:ext_results}.}
  \label{fig:ext_results}
\end{figure*}

% ---------------------------------------------------------------------------- %

\subsection{Optimising Tonal Data}
\label{sec:ext_tonal}

We have seen that for homogeneous diffusion inpainting, an additional tonal 
optimisation yields significant improvements in the reconstruction quality. 
Thus, we should also investigate if it is beneficial for biharmonic and EED 
inpainting.\par{}

In the case of biharmonic inpainting, our tonal optimisation strategy
from Subsection \ref{sec:optim-tonal-data} carries over literally, 
since we are still facing a linear least squares problem: For a fixed 
inpainting mask, the reconstruction depends linearly of the specified 
grey values.
Thus, all we have to do is to exchange the discrete differential operator 
for homogeneous diffusion inpainting by its biharmonic counterpart. 
Also our algorithms such as the FED-accelerated gradient descent remain 
applicable. The final results for biharmonic inpainting with spatial 
and tonal optimisation are listed in Tab.~\ref{tab:ext_results}, 
and the best reconstruction is depicted in Fig.~\ref{fig:ext_results}(b). 
As one can see, tonal optimisation allows to reduce the MSE from $20.89$ 
to $16.73$.\par{}

Since our tonal optimisation methods are tailored towards linear least 
squares formulations, specific challenges arise when we want to extend 
them to EED-based inpainting. The nonlinearity of the EED inpainting 
scheme prevents closed form expressions such as \eqref{eq:2} and 
\eqref{eq:3}. This is caused by the fact that the matrix $\bm{A}$ is
now a function of the inpainted image $\bm{u}^k$, which itself depends 
on the grey value data $\bm{g}^k$ in the specified pixels. Although
this concatenated mapping is formally smooth, one should keep in mind
that EED has the ability to create edge-like structures. This means
that in practice the problem is fairly ill-conditioned: Small local
changes in $\bm{g}^k$ may have a strong global impact on $\bm{u}^k$, 
on $\bm{A}(\bm{u}^k)$, and on the Jacobian of the error function 
(\ref{eq:error}).
Moreover, there is no guarantee anymore that the tonal optimisation
problem is strictly convex. Thus, it may have many local minimisers. 
It is therefore not surprising that different algorithms and even 
different parameter settings within the same algorithm may end up 
in different minimisers.  
Since it is difficult to design practically feasable algorithms that
guarantee to find the global minimiser, we focus on transparent and 
conceptually simple local optimisation strategies such as gradient 
descent.
We have done a number of experiments with several variants of gradient 
descent approaches for tonal optimisation in EED-based inpainting. 
Below we describe the method that has yielded the best results in terms 
of reconstruction quality.

We suggest to modify the gradient descent approach for \eqref{eq:error} 
as follows. While we keep the basic structure of \eqref{eq:gd} and 
\eqref{eq:2}, we lack a closed form solution of type
\eqref{eq:3} for the Jacobian $\bm{J}$ that is now an unknown 
function of the evolving grey value data $\bm{g}^k$. As a remedy, we 
approximate the Jacobian by numerical differentiation:
\begin{equation}
  \label{eq:4}
  (\bm{J}(\bm{g}^{k}))_{i,j} :=
  \frac{\inp{\, \bm{g}^{k}\!+ \eta \, \bm{e_j}}_{i}-
    \inp{\bm{g}^{k}}_{i}}{\eta},
\end{equation}
where $i$ denotes the pixel index, $j$ is the index of an individual 
mask point, and the parameter $\eta>0$ quantifies a small grey value
perturbation at the mask point $j$. As before, $\bm{e_j}$ is the
canonical basis vector with a unit impulse in pixel $j$. Note that 
the grey values at all non-mask points do not have any influence on 
the inpainting result. 
Thus, the Jacobian does not have to be computed for those locations.

In contrast to our linear grey value optimisation strategies,
we abstain from adapting the gradient descent step size $\alpha$ 
by means of exact line search or cyclic FED-based variations: We 
have experienced that the high sensitivity of the Jacobian w.r.t. 
$\bm{g}^{k}$ suggests to keep the gradient descent step size $\alpha$ 
fairly small to capture this dynamics in an adequate way. On the other 
hand, too small values for $\alpha$ can also result in a convergence 
towards a fairly poor local minimiser. Thus, in practice, one fixes
$\alpha$ to some small value.

Similar considerations apply to the parameter $\eta$ in \eqref{eq:4}:
The nonlinear dynamics of the Jacobian suggests to use small values for
$\eta$ in order to approximate the derivative well. On the other side, 
if $\eta$ is too small, numerical problems can arise, since both the 
numerator and the denominator in \eqref{eq:4} approach zero.

Table \ref{tab:EED-GVO-MSE} shows how an appropriate selection of the
parameters $\alpha$ and $\eta$ can be used to optimise the reconstruction 
quality for EED-based inpainting. We observe that in our test scenario 
the resulting tonal optimisation step allows to reduce the MSE from 
$12.62$ to $10.79$. The corresponding reconstruction is depiced
in Fig.~\ref{fig:ext_results}(d).
Tab.~\ref{tab:ext_results} shows that the MSE of $10.79$ obtained for 
EED is far better than the errors for homogeneous diffusion inpainting 
(MSE=$27.24$) and for biharmonic inpainting (MSE=$16.73$). To the best 
of our knowledge, such a reconstuction quality has never been achieved 
before for PDE-based inpainting with a mask density of only $4 \%$. 

%............................................................................
 
\begin{table}
 \centering
 \caption{
   MSE obtained after a tonal optimisation for the EED inpainting approach 
   with the final mask from Fig.~\ref{fig:ext_results}. The parameter 
   $\alpha$ denotes the fixed step size in each gradient descent iteration, 
   and $\eta$ is the step size used for the computation of the derivatives 
   by means of finite differences.}
 \label{tab:EED-GVO-MSE}
 \begin{tabular}{@{\extracolsep{\fill}}ccc@{\extracolsep{\fill}}}\toprule
   \multirow{2}{*}[-1mm]{$\eta$} & 
   \multicolumn{2}{c}{$\alpha$} \\
   \cmidrule{2-3} & $10^{-3}$ & $10^{-2}$ \\ 
   \cmidrule(r){1-1} \cmidrule(lr){2-2} \cmidrule(l){3-3}
   \addlinespace
   1.0 & 10.86 & \textbf{10.79} \\
   1.5 & 10.93 & 10.80 \\
   2.0 & 10.88 & 10.80 \\
   3.0 & 10.88 & 10.81 \\ 
   \addlinespace\bottomrule
 \end{tabular}
\end{table}

%%%%%%%%%%%%%%%%%%%%%%%%%%%%%%%%%%%%%%%%%%%%%%%%%%%%%%%%%%%%%%%%%%%%%%%%%%%%%

\section{Summary and Conclusions}
\label{sec:conclusion}

The contributions of our work are twofold: On the one hand we gave 
a comprehensive survey on the state-of-the-art in PDE-based 
image compression. It enables the readers to obtain an overview 
of the achievements that have been made in this field and guide 
them to the relevant literature.

On the other hand, we were focusing on one key problem in PDE-based image
compression: the spatial and tonal data optimisation. Here we aimed at
obtaining insights into the full potential of the quality gains, without 
imposing any constraints on the type of inpainting operator, 
the computational costs, or the coding costs for the optimised data.
Our strategy was to start with more restricted settings that allow
a deeper understanding of the problem, and to extend our findings to
more general scenarios afterwards.

In the 1D setting, we restricted ourselves to homogeneous diffusion
impainting of strictly convex functions. For the resulting free knot
problem for linear spline interpolation, we came up with a new algorithm
and discussed optimal approximation results. We showed the nonconvexity
of the problem for more than three knots, and we saw that a splitting 
into a spatial optimisation step followed by a tonal one does not 
deteriorate the reconstruction quality substantially.

This motivated us to use such a splitting strategy also for the 2D 
setting and without restrictions on the convexity or concavity of the
image data. For spatial data optimisation with homogeneous diffusion
inpainting, we studied a probabilistic sparsification approach with 
nonlocal pixel exchange in detail. For the subsequent tonal optimisation 
step we established the existence of a unique solution. We showed that 
it can be found efficiently with a gradient descent approach which is 
accelerated by a fast explicit diffusion (FED) strategy.

In our investigations we were aiming at fairly generic algorithms 
that can also be applied to other inpainting operators such as 
biharmonic inpainting and inpainting based on edge-enhancing 
anisotropic diffusion (EED). Data optimisation for EED-based 
inpainting allowed us to come up with reconstructions of hitherto 
unparalleled quality.

Our framework for data optimisation permits to specify the desired mask 
density {\em a priori}. This is a conceptual advantage over approaches 
that have to tune a regularisation parameter in order to influence the 
density which results {\em a posteriori} \cite{HSW13,CRP14}. In practice 
this means that the latter approaches may have to run the algorithm 
many times.

While our results show the large potential of PDE-based inpainting, it 
should be emphasised that data optimisation is a key problem, but not
the only one that must be solved for building well-performing codecs: 
For example, both the location and the grey values of the data must be 
stored efficiently, and some real-world applications require very fast 
algorithms. This poses additional constraints and challenges that will 
be addressed in our forthcoming publications.

%----------------------------------------------------------------------------

\paragraph{Acknowledgements}
 Our research is partly funded by the Deutsche Forschungsgemeinschaft (DFG)
 through a Gottfried Wilhelm Leibniz Prize. This is gratefully acknowledged. 
 We also thank Pascal Gwosdek and Christian Schmaltz for providing the
 electrostatic halftoning images for us.

%%%%%%%%%%%%%%%%%%%%%%%%%%%%%%%%%%%%%%%%%%%%%%%%%%%%%%%%%%%%%%%%%%%%%%%%%%%%%

\bibliographystyle{spmpsci}
\bibliography{hoeltgen}

%%%%%%%%%%%%%%%%%%%%%%%%%%%%%%%%%%%%%%%%%%%%%%%%%%%%%%%%%%%%%%%%%%%%%%%%%%%%%

\end{document}